\documentclass{IEEEtaes}
\usepackage{color,array,amsthm}
\usepackage{amsmath,amsfonts}
\usepackage{algorithm}
\usepackage{algpseudocode} 
\usepackage[caption=false,font=footnotesize]{subfig}
\usepackage{textcomp}
\usepackage{stfloats}
\usepackage{url}
\usepackage{verbatim}
\usepackage{graphicx}
\graphicspath{{figs_TAES2/}{./}}
\usepackage{cite}
\usepackage{booktabs}
\usepackage{longtable}
\usepackage{multirow}
\usepackage{mathtools,amssymb}
\usepackage{siunitx}
\usepackage{tabularx}
\usepackage{makecell} 
\newcolumntype{Y}{>{\raggedright\arraybackslash}X} 
\usepackage{bm}

\usepackage{commath}

\newtheorem{theorem}{Theorem}

\jvol{XX}
\jnum{XX}
\jmonth{XXXXX}
\paper{1234567}
\pubyear{2022}
\doiinfo{TAES.2022.Doi Number}
\begin{document}
\title{Trajectory Planning for a Multi-UAV Rigid-Payload Cascaded Transportation System Based on Enhanced Tube-RRT*}
\author{Jianqiao Yu}
\member{Member, IEEE}
\author{Jia Li}
\member{Member, IEEE}
\affil{Beijing Institute of Technology, Beijing , China}
\author{Tianhua Gao}
\member{Member, IEEE}
\affil{University of Tsukuba, Ibaraki, Japan}
\authoraddress{
  Authors’ addresses: Jianqiao Yu and Jia Li are with the School of Aerospace Science and Technology, Beijing Institute of Technology, Beijing 100081, China.
 (e-mail: jianqiao@bit.edu.cn; 3120240042@bit.edu.cn). 
 T. Gao is with the Graduate School of Systems and Information Engineering, University of Tsukuba, Japan. (e-mail: gao.tianhua.tkb\_gb@u.tsukuba.ac.jp).
(Corresponding author: T. Gao.)}
\supplementary{For the trajectory planning and tracking of a multi-UAV rigid payload cascade transportation system, see supplementary materials.
This work has been submitted to the IEEE for possible publication. Copyright may be transferred without notice, after which this version may no longer be accessible.}
\maketitle
\begin{abstract}This paper presents a two-stage trajectory planning framework for a multi-UAV rigid-payload cascaded transportation system, 
  aiming to address planning challenges in densely cluttered environments.
In Stage~I, an Enhanced Tube-RRT* algorithm is developed by integrating active hybrid sampling and an adaptive expansion strategy, enabling rapid generation of a safe and feasible virtual tube in environments with dense obstacles.
Moreover, a trajectory smoothness cost is explicitly incorporated into the edge cost to reduce excessive turns and thereby mitigate cable-induced oscillations.
Simulation results demonstrate that the proposed Enhanced Tube-RRT* achieves a higher success rate and effective sampling rate than mixed-sampling Tube-RRT* (STube-RRT*) and 
adaptive-extension Tube-RRT*(AETube-RRT*), while producing a shorter optimal path with a smaller cumulative turning angle.
In Stage~II, a convex quadratic program is formulated by considering payload translational and rotational dynamics, cable tension constraints, and collision-safety constraints, yielding a smooth, collision-free desired payload trajectory.
Finally, a centralized geometric control scheme is applied to the cascaded system to validate the effectiveness and feasibility of the proposed planning framework, offering a practical solution for payload attitude maneuvering in densely cluttered environments.
\end{abstract}
\begin{IEEEkeywords}
Centralized geometric control, multi-UAV rigid-payload cascaded transportation system, obstacle avoidance, trajectory planning
\end{IEEEkeywords}
\section{INTRODUCTION}
With the continuous increase in mission complexity and payload requirements, cooperative aerial transportation systems composed of multiple unmanned aerial vehicles (UAVs) \cite{10758214,10780984,11036553} 
have gradually emerged as an important technological approach to overcome the limitations of payload capacity and maneuverability inherent in single-UAV platforms.
Such systems not only significantly enhance the overall payload capability, but also provide improved maneuvering flexibility \cite{cbsystems} 
and mission adaptability under heavy-load and high-inertia conditions.The core challenge of motion planning for these systems lies in generating executable trajectories that are safe, 
smooth, and energy-efficient under the combined influence of multiple safety constraints, system dynamics constraints, and actuator limitations \cite{10962902}.

Due to the highly underactuated nature of cable-suspended payload transportation systems and their pronounced strongly coupled nonlinear dynamics, a substantial body of research has proposed representative trajectory planning schemes by leveraging explicit system dynamic models.
In \cite{Sreenath2013DynamicsCA}, the authors first proved that cooperative transportation systems with multiple quadrotors satisfy the differential flatness condition, and then parameterized the system states and control inputs using flat outputs to generate smooth trajectories with high-order continuity via optimization.
Although differential flatness enables an explicit parameterization of the trajectory generation problem and thus converts it into an algebraic planning problem, this structured parameterization becomes increasingly difficult to satisfy multiple performance objectives and feasibility requirements as the system scale grows and additional complex conditions (e.g., tension constraints, obstacle avoidance, and input limits) are introduced.
Consequently, more studies have shifted to continuous optimization frameworks based on optimal control \cite{10478625,8962254,Chai2025TMechTrajPlan,10325595,CHAN2026111460}.

In \cite{10325595}, a direct collocation--based optimal trajectory planning framework was adopted, in which the coupled UAV--payload dynamics, input constraints, and obstacle-avoidance constraints were incorporated into a unified optimization problem, enabling the simultaneous treatment of multiple constraint classes at the optimization level.
In \cite{CHAN2026111460}, the space was reconstructed under a gravity-normal (GN) formulation, and a direct multiple shooting method was employed to generate globally time-optimal trajectories, which simplified the expression of spatial constraints and resulted in shorter flight time.
Since the performance of continuous optimal control methods is often sensitive to the quality of the initial guess, this has motivated the integration of planning frameworks with global search capability.
In \cite{Chai2025TimeoptimalGT}, the original model was simplified to a second-order integrator as the dynamic constraint, and time-optimal trajectories were generated in the sense of the system's bounding sphere.
In \cite{11358534}, PolyFly was proposed as a globally optimal planning framework by modeling the quadrotor, cable, payload, and environment as orientation-aware polyhedra and constructing differentiable constraints via duality theory, thereby enabling fast optimal trajectory generation with a non-conservative geometric representation.
In \cite{11153080}, the pc-dbCBS planner was proposed, which achieves planning with dynamic constraints at arbitrary times by combining three-layer conflict detection/resolution and state stacking, together with alternating representations between stacked states and minimal coordinates.
In \cite{11245968}, a MINCO-based optimization framework was developed by integrating rotational formation regulation and GeoSafe safe-region expansion, overcoming the limitations of conventional four-DoF planning in ultra-narrow spaces and enabling efficient and scalable passage for multi-UAV payload systems.
In \cite{10144367}, an online formation planning method was proposed for cooperative transportation by multiple multirotor UAVs with unknown payload mass and cable length; by analyzing cable tension, the authors designed formation criteria and an optimization-based performance function, and generated optimal trajectories in conjunction with an admittance model.

Although global optimization can theoretically mitigate local minima, its computational complexity typically grows rapidly with the problem scale.
In recent years, hybrid planning frameworks that combine search-based methods \cite{10243043,10938620,1365,11123816} or sampling-based methods \cite{8869298,10802794,11030654} with local optimization have been widely adopted.
In \cite{11123816}, a heuristic-guided 3-D motion-primitive graph search algorithm was proposed to improve trajectory smoothness, and the cable tension was further optimized by incorporating the payload translational dynamics.
In \cite{10802794}, the payload--cable states produced by a geometric sampling strategy were used as the initial guess, and a full-dynamics motion planning problem with nonlinear trajectory optimization was solved, including all state variables such as the UAV attitude; a high success rate was demonstrated experimentally.
In \cite{11030654}, a 3-D trajectory planning method combining the BIGIT* algorithm with minimum-snap polynomial optimization was proposed: collision-free waypoint sequences were first generated, and then smooth and feasible trajectories were constructed and integrated with a two-layer explicit model predictive control framework, enabling real-time stabilized quadrotor flight control under constraints and uncertainties.

Despite the encouraging progress under complex constraints, the aforementioned planning approaches still suffer from two major limitations:
\emph{i}) for cooperative aerial transportation systems adopting distributed architectures, the payload attitude may remain uncontrolled during transportation \cite{Gao2025RobustnessEF}, which can cause attitude errors to accumulate during aggressive maneuvers;
\emph{ii}) many methods only enforce constraints on the payload translational dynamics while neglecting the coupling mechanism between attitude and cable tension, making it difficult to guarantee physical feasibility in both geometric and tension-related senses.

To reduce reliance on accurate models, learning-based trajectory generation methods have also been explored \cite{10711789,10268605,LI2021106887,11373839}.
For example, \cite{LI2021106887} developed a reinforcement learning framework that does not require an accurate dynamics model; trajectory generation was realized via value-function approximation and policy iteration, producing smooth transportation paths that suppress payload swing.
In \cite{11373839}, a reinforcement learning policy enabled high-speed and stable flight of a quadrotor slung-load system in designed scenarios, avoiding the high computational burden of traditional optimization methods arising from cable mode switching and complex dynamics modeling.
Moreover, several studies have investigated cooperative transportation from an integrated planning--control perspective.
In \cite{10178340}, higher-order time-varying control barrier functions (HOCBFs) were introduced to enforce complex Signal Temporal Logic specifications, and transportation tasks were executed through a finite set of discrete waypoints.
In \cite{CAI2025110255}, a new co-design strategy for the planner and controller was proposed, where an event-triggered real-time planner can effectively suppress payload swing.
In \cite{doi:10.1126/scirobotics.adu8015}, a centralized planner was used in a laboratory environment with predefined obstacle/no-fly regions to achieve highly agile payload pose tracking and online obstacle avoidance, extending the planning and control capability of transportation systems in controlled environments.

While these studies have further advanced planning methods for aerial transportation systems, systematic solutions remain lacking for trajectory planning of cascaded aerial transportation systems operating in outdoor environments with complex obstacles, especially under dense obstacle conditions:
\emph{i}) the volumes of outdoor obstacles and transportation systems are non-negligible, making it difficult to express safety constraints using point-mass models or overly conservative inflation approximations;
\emph{ii}) the cascaded structure introduces stronger geometric occupancy and motion coupling, increasing the dimensionality of system--environment collision checking and feasible-trajectory construction, thereby limiting the scalability of conventional methods in complex environments.

Motivated by these challenges, this paper focuses on the trajectory planning problem of a multi-UAV cascaded transportation system with rigid payloads in outdoor environments with dense obstacles, while explicitly accounting for the impact of both obstacle volume and system volume on safety constraints.
The objective is to achieve a better balance among safety, efficiency, and computational complexity, and to provide a generalizable planning framework and implementation pathway for reliable aerial transportation in complex environments.
Specifically, for a multi-UAV cascaded transportation system with rigid payloads built upon a centralized control strategy \cite{Gao2025RobustnessEF}, we explicitly incorporate dynamic constraints---in particular, payload attitude variations and cable-tension variations---within a two-stage planning scheme.
First, an Enhanced Tube-RRT* algorithm is developed based on the Tube-RRT* framework, where proactive mixed sampling and adaptive expansion strategies are introduced to improve feasible-solution discovery efficiency in dense-obstacle environments, enabling sampling-based search to consistently output collision-free feasible paths with safety margins under complex geometric constraints.
Furthermore, we propose a payload-oriented trajectory planning method that incorporates cable-tension feasibility and payload dynamics as explicit constraints, ensuring consistency between the planned trajectories and the subsequent centralized tracking controller at both the dynamics and constraint levels.
This reduces the risk of infeasibility caused by planning--control mismatch and improves the overall system safety and robustness.

The main contributions of this paper are summarized as follows:
\begin{enumerate}
  \item To the best of our knowledge, this paper presents the first two-stage obstacle-avoidance trajectory planning framework for multi-UAV cascaded transportation systems with rigid payloads, in which payload attitude variations and system-volume constraints are explicitly considered during planning.
  \item We develop Enhanced Tube-RRT* based on Tube-RRT*, where proactive mixed sampling and adaptive expansion strategies improve the search efficiency and solution quality in narrow passages and dense-obstacle scenarios, enabling rapid generation of collision-free feasible paths.
  \item A centralized tracking control strategy is employed to achieve dynamic tracking of the payload trajectory and attitude, and simulations validate the feasibility and effectiveness of the proposed planning--control scheme.
\end{enumerate}

The remainder of this paper is organized as follows.
Section~II presents the dynamic model of the cascaded aerial transportation system and formulates the problem.
Section~III introduces the proposed Enhanced Tube-RRT* algorithm.
Section~IV formulates the trajectory optimization problem for the cascaded transportation system.
Section~V reports simulation results and verifies trajectory feasibility using centralized geometric control.
Finally, Section~VI concludes the paper.

\textbf{Notation}
\emph{
The Euclidean space of dimension $n$ is denoted by $\mathbb{R}^n$.
The special orthogonal group is denoted by $SO(3)$ and its Lie algebra
by $\mathfrak{so}(3)$. The two–sphere is denoted by $\mathbb{S}^2$.
For $x\in\mathbb{R}^3$, the hat map $(\cdot)^\wedge:\mathbb{R}^3
\rightarrow\mathfrak{so}(3)$ satisfies $\hat{x}y=x\times y$ for any
$x,y\in\mathbb{R}^3$, where $\times$ denotes the cross product.The superscript $(k)$ denotes the $k$-th order time derivative with $k>0$.
The notation $\|\cdot\|$ denotes the Euclidean norm and $\mathbf{I}$ denotes the
identity matrix.}
\section{Problem Statement and System Modeling}
\subsection{Problem Formulation}
\label{subsec:problem_formulation}
\begin{figure}[htbp]
  \centering
  \includegraphics[width=0.8\linewidth]{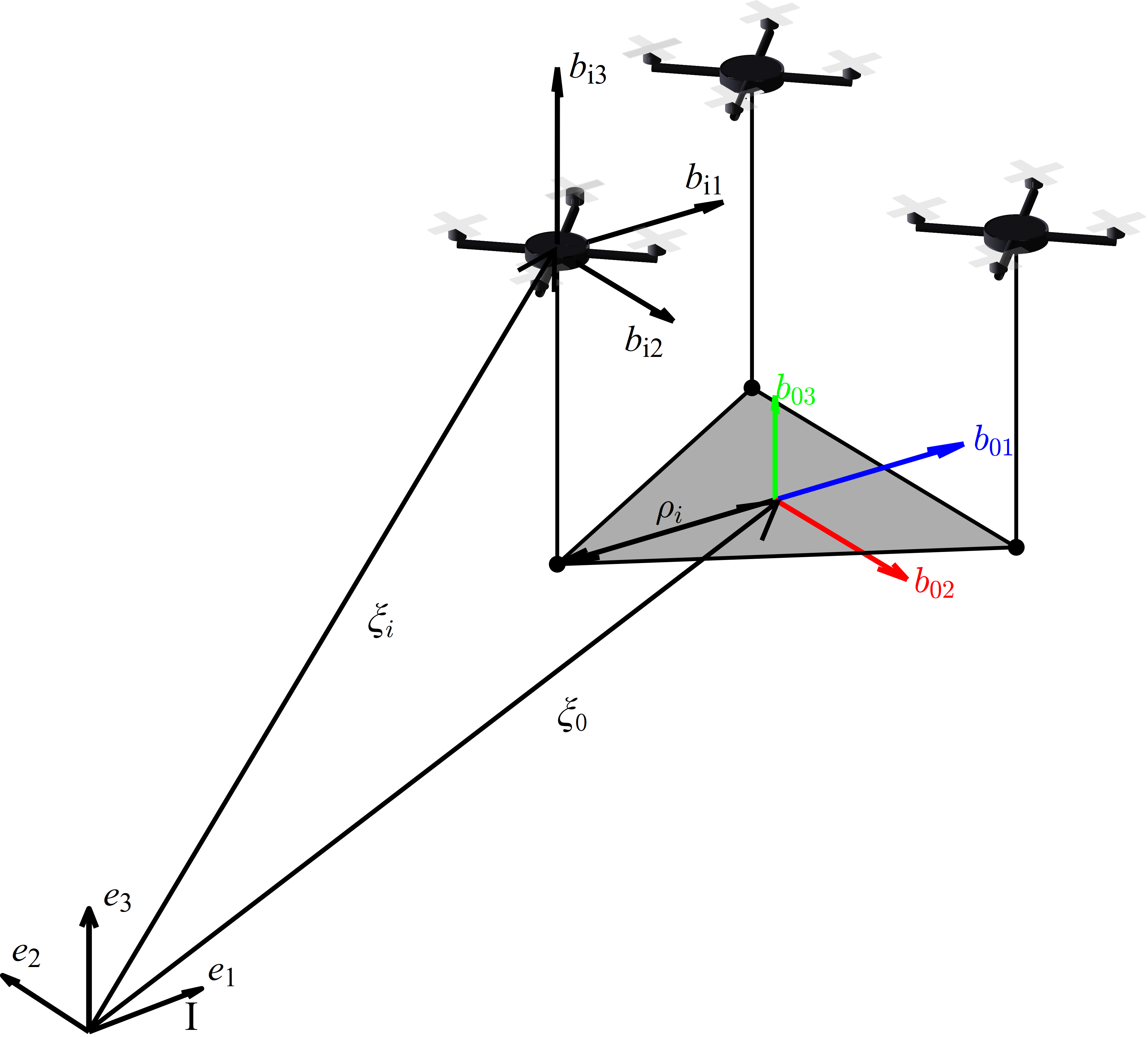}
  \caption{Schematic of the multi-UAV tethered payload transportation system.}
  \label{fig:system}
\end{figure}

Consider a system in which $n_{C}$ unmanned aerial vehicles (UAVs) cooperatively transport a rigid payload using cables (in this paper $n_{C}=3$), as illustrated in Fig.~\ref{fig:system}.
Let the load-carrying UAVs be indexed by nodes $i\in\{1,\dots,n_{C}\}$.
Each node $i$ is connected to a corner of the payload via an inextensible and massless cable with length $l_{\mathrm{i}}>0$.

The system evolves in a bounded three-dimensional workspace $\mathcal{X}\subset\mathbb{R}^3$ containing a set of known static obstacles
$\mathcal{O} = \{\mathcal{O}_k\}_{k=1}^{N_o}$.
The corresponding free space is defined as
$\mathcal{X}_{\mathrm{free}}=\mathcal{X}\setminus \mathcal{O}$.

Given the initial system state $s(0)=s_0$, the initial payload position $\boldsymbol{\xi}_{\mathrm{o}}\in\mathbb{R}^3$, and the goal position $\boldsymbol{\xi}_{\mathrm{g}}\in\mathbb{R}^3$,
the objective of planning is to generate a smooth trajectory $\boldsymbol{\xi}^d\in\mathbb{R}^3$ for the payload that accounts for both obstacle avoidance and system-volume constraints, while satisfying the payload dynamic constraints and cable-tension feasibility.

In summary, this paper addresses the following problem:
\emph{for a multi-UAV cascaded transportation system carrying a rigid payload operating in a dense-obstacle three-dimensional environment with system parameter constraints, design a trajectory planning method that transports the payload from the initial position to the target position by generating a smooth desired trajectory satisfying the cascaded geometric constraints, dynamic constraints, and cable-tension constraints, while ensuring collision avoidance among the payload, UAVs, cables, and obstacles.
The resulting trajectory serves as the reference input for subsequent trajectory tracking using a centralized geometric control strategy.}
\subsection{System Model}
Let the position and attitude of the payload with respect to the inertial frame $\mathrm{I}$ be denoted by $\boldsymbol{\xi}_0 \in \mathbb{R}^3$ and $\mathbf{R}_0 \in \mathrm{SO}(3)$, respectively. 
The direction of the $i$th cable is represented by $\boldsymbol{q}_i \in \mathbb{S}^2$, and $\boldsymbol{\rho}_i \in \mathbb{R}^3$ denotes the position of the corresponding attachment point with respect to the payload center of mass. 
Let $\mathbf{R}_i \in \mathrm{SO}(3)$ denote the attitude of the $i$th UAV. 
Then, the position of the $i$th UAV, denoted by $\boldsymbol{\xi}_i \in \mathbb{R}^3$, is given by
\begin{equation}
\boldsymbol{\xi}_i
=
\boldsymbol{\xi}_0
+
\mathbf{R}_0 \boldsymbol{\rho}_i
+
l_i \boldsymbol{q}_i .
\end{equation}

Based on Lagrangian mechanics formulated on nonlinear manifolds, the coordinate-free equations of motion can be derived as follows \cite{7843619}:
\begin{equation}
\begin{aligned}
&m_0(\ddot{\boldsymbol{\xi}}_0-\boldsymbol{g})
+\sum_{i=1}^{n_C} m_i \boldsymbol{q}_i \boldsymbol{q}_i^{\top}
\left(
\ddot{\boldsymbol{\xi}}_0
-\boldsymbol{g}
-\mathbf{R}_0\widehat{\boldsymbol{\rho}}_i\dot{\boldsymbol{\Omega}}_0
\right) \\
&=
\Delta_{\xi_0}
+\sum_{i=1}^{n_C}
\Big(
\boldsymbol{u}_i
+\Delta_{\xi_i}
-m_i l_i\|\boldsymbol{\omega}_i\|^{2}\boldsymbol{q}_i \\
&\qquad
-m_i\boldsymbol{q}_i\boldsymbol{q}_i^{\top}
\mathbf{R}_0
\widehat{\boldsymbol{\Omega}}_0^{2}
\boldsymbol{\rho}_i
\Big).
\end{aligned}
\end{equation}
\begin{equation}
\begin{aligned}
&\mathbf{J}_0\dot{\boldsymbol{\Omega}}_0
+\widehat{\boldsymbol{\Omega}}_0
\mathbf{J}_0
\boldsymbol{\Omega}_0 \\
&+
\sum_{i=1}^{n_C}
m_i\widehat{\boldsymbol{\rho}}_i
\mathbf{R}_0^{\top}
\boldsymbol{q}_i\boldsymbol{q}_i^{\top}
\left(
\ddot{\boldsymbol{\xi}}_0
-\boldsymbol{g}
-\mathbf{R}_0\widehat{\boldsymbol{\rho}}_i\dot{\boldsymbol{\Omega}}_0
\right) \\
&=
\Delta_{R_0}
+\sum_{i=1}^{n_C}
\Big(
\widehat{\boldsymbol{\rho}}_i
\mathbf{R}_0^{\top}\boldsymbol{u}_i
+\Delta_{\xi_i} \\
&\qquad
-m_i l_i\|\boldsymbol{\omega}_i\|^{2}\boldsymbol{q}_i
-m_i\boldsymbol{q}_i\boldsymbol{q}_i^{\top}
\mathbf{R}_0
\widehat{\boldsymbol{\Omega}}_0^{2}
\boldsymbol{\rho}_i
\Big).
\end{aligned}
\end{equation}

\begin{equation}
\begin{aligned}
m_i l_i \dot{\boldsymbol{\omega}}_i
&=
m_i\widehat{\boldsymbol{q}}_i^{\top}
\Big(
\ddot{\boldsymbol{\xi}}_0
-\boldsymbol{g}
-\mathbf{R}_0\widehat{\boldsymbol{\rho}}_i\dot{\boldsymbol{\Omega}}_0 \\
&\quad
+\mathbf{R}_0\widehat{\boldsymbol{\Omega}}_0^{2}\boldsymbol{\rho}_i
\Big)
-\widehat{\boldsymbol{q}}_i^{\top}
\big(
\boldsymbol{u}_i^{\perp}
+\Delta_{\xi_i}^{\perp}
\big).
\end{aligned}
\end{equation}

\begin{equation}
\mathbf{J}_i\dot{\boldsymbol{\Omega}}_i
+
\widehat{\boldsymbol{\Omega}}_i
\mathbf{J}_i
\boldsymbol{\Omega}_i
=
\boldsymbol{M}_i
+
\Delta_{R_i}.
\end{equation}
Here, $m_0$ denotes the mass of the payload, and $m_i$ denotes the mass of the $i$th UAV.
$\mathbf{J}_0$ represents the inertia matrix of the payload, while $\mathbf{J}_i$ denotes the rotational inertia matrix of the $i$th UAV.
The angular velocity of the payload is denoted by $\boldsymbol{\Omega}_0\in\mathbb{R}^3$.
The angular velocity of the $i$th cable is denoted by $\boldsymbol{\omega}_i\in\mathbb{R}^3$, and the angular velocity of the $i$th UAV is denoted by $\boldsymbol{\Omega}_i\in\mathbb{R}^3$.
The control thrust generated by the $i$th UAV is denoted by $f_i\in\mathbb{R}$, and the control moment applied to the $i$th UAV is denoted by $\boldsymbol{M}_i\in\mathbb{R}^3$.
The gravitational acceleration vector is denoted by $\boldsymbol{g}\in\mathbb{R}^3$.
The disturbance terms $\Delta_{\xi_0}\in\mathbb{R}^3$ and $\Delta_{R_0}\in\mathbb{R}^3$ represent the uncertainties and disturbances acting on the translational and rotational dynamics of the payload, respectively.

The vector $\boldsymbol{u}_i\in\mathbb{R}^3$ denotes the total thrust generated by the $i$th quadrotor, defined as
\begin{equation}
\boldsymbol{u}_i = -f_i \mathbf{R}_i \boldsymbol{e}_3 ,
\end{equation}
where $\boldsymbol{e}_3=[0,0,1]^{\top}$ is the unit vector along the third axis of the inertial frame $\mathrm{I}$.

In this paper, a hierarchical cascaded control framework is adopted.
The payload controller generates the desired thrust component parallel to the cable, denoted by $\boldsymbol{u}_i^{\parallel}$,
while the perpendicular component $\boldsymbol{u}_i^{\perp}$ is determined by the cable dynamics.
They are defined as
\begin{equation}
\boldsymbol{u}_i^{\parallel}
=
\boldsymbol{q}_i\boldsymbol{q}_i^{\top}\boldsymbol{u}_i
\end{equation}
\begin{equation}
\boldsymbol{u}_i^{\perp}
=
-\widehat{\boldsymbol{q}}_i^{2}\boldsymbol{u}_i
=
(\mathbf{I}-\boldsymbol{q}_i\boldsymbol{q}_i^{\top})\boldsymbol{u}_i .
\end{equation}

Similarly, the disturbance $\Delta_{\xi_i}\in\mathbb{R}^3$ acting on the $i$th UAV can be decomposed as
\begin{equation}
\Delta_{\xi_i}^{\parallel}
=
\boldsymbol{q}_i\boldsymbol{q}_i^{\top}\Delta_{\xi_i},
\end{equation}
\begin{equation}
\Delta_{\xi_i}^{\perp}
=
-\widehat{\boldsymbol{q}}_i^{2}\Delta_{\xi_i}
=
(\mathbf{I}-\boldsymbol{q}_i\boldsymbol{q}_i^{\top})\Delta_{\xi_i}.
\end{equation}

\section{Enhanced Tube-RRT$^\ast$}
This section presents an Enhanced Tube-RRT$^\ast$ algorithm tailored for load-transport scenarios. The method focuses on planning techniques that improve trajectory quality. These techniques include active mixed sampling, an adaptive extension strategy, and a composite cost function. The goal is to maximize trajectory smoothness. This work also provides a foundation for subsequent trajectory generation and control of the overall system.

\subsection{Active Mixed Sampling}
First, a Voronoi diagram is introduced to partition the configuration space $\mathcal{X}$ into $K$ subregions.
Each subregion is denoted as $\mathcal{X}_k\,(k=1,\dots,K)$ and defined as
\begin{equation}
\mathcal{X}_k=
\left\{
\boldsymbol{\xi} \in \mathcal{X}\ \bigg|\ 
\left\lVert \boldsymbol{\xi}-\boldsymbol{r}_i\right\rVert \le
\left\lVert \boldsymbol{\xi}-\boldsymbol{r}_j\right\rVert,\ \forall\, i\neq j
\right\},
\end{equation}
where $\boldsymbol{\xi} \in \mathbb{R}^3$ denotes an arbitrary point in the
workspace, $\boldsymbol{r}_i, \boldsymbol{r}_j \in \mathbb{R}^3$ denote two distinct generator points.

To avoid the sampling instability caused by fixed Gaussian distribution parameters in \cite{Ning2026TSRIL}, a Beta posterior uncertainty model is introduced to update the sampling probability of each subregion.
Thompson sampling is further employed to actively select subregions, eliminating the need for additional heuristic weights or manual parameter tuning.

Let $\theta_k\in(0,1)$ denote the sampling–expansion success probability of subregion $\mathcal{X}_k$. The prior distribution is defined as
\begin{equation}
\begin{aligned}
  \theta_k &\sim \mathrm{Beta}\!\left(\alpha_k^{(0)},\,\beta_k^{(0)}\right),\\
\alpha_k^{(0)} &= 1+\kappa_p\left(1-\rho_k^{\mathrm{obs}}\right),\\
\beta_k^{(0)}  &= 1+\kappa_p\,\rho_k^{\mathrm{obs}},
\end{aligned}
\end{equation}
where positive constant $\kappa_p$ denotes the strength of the Beta prior, and $\rho_k^{\mathrm{obs}}\in[0,1]$ denotes the estimated obstacle density of the subregion, which can be obtained via a Monte Carlo method \cite{Ning2026TSRIL}.
As the obstacle density increases, the prior mean 
$\mathbb{E}[\theta_k] = \frac{\alpha_k^{(0)}}{\alpha_k^{(0)}+\beta_k^{(0)}}$ 
becomes smaller, indicating a lower probability of feasible expansion.

Each expansion attempt within a subregion is treated as a Bernoulli trial.
At each iteration, a sample of the success probability is drawn from the posterior distribution of every subregion, and the subregion with the largest sampled value is selected for uniform sampling:
\begin{align}
\hat{\theta}_k^{(t)} &\sim \operatorname{Beta}\!\left(\alpha_k^{(t)},\,\beta_k^{(t)}\right), \quad k=1,\ldots,K
\label{eq:beta_sample}\\
k^{\ast} &= \arg\max_{k\in\{1,\ldots,K\}} \hat{\theta}_k^{(t)},
\label{eq:best_region}\\
\boldsymbol{\xi}_{\mathrm{rand}} &\sim \operatorname{Unif}\!\left(\mathcal{X}_{k^{\ast}}\right).
\label{eq:sampling}
\end{align}

Once an observation $y_t$ is obtained within $\mathcal{X}_{k^{\ast}}$, the posterior parameters $\alpha_{k^\ast}$ and $\beta_{k^\ast}$ are updated online as
\begin{equation}
\begin{split}
\alpha_{k^*} \leftarrow \alpha_{k^*}+y_t, \\
\beta_{k^*} \leftarrow \beta_{k^*}+(1-y_t),
\end{split}
\label{eq:beta_update}
\end{equation}
where $y_t$ denotes the binary feedback obtained during the $t$-th sampling and expansion attempt in subregion $\mathcal{X}_k$:
\begin{equation}
y_t=
\begin{cases}
1, & \text{if the generated node is feasible},\\
0, & \text{otherwise}.
\end{cases}
\end{equation}

Furthermore, to improve robustness in complex environments and avoid losing exploration capability due to excessive bias (e.g., missing narrow passages), the following mixed sampling strategy is adopted:
\begin{equation}
p(\boldsymbol{\xi})=\epsilon_u\,u(\boldsymbol{\xi})+(1-\epsilon_u)\big[p_g\,\delta(\boldsymbol{\xi}-\boldsymbol{\xi}_g)+(1-p_g)\,q(\boldsymbol{\xi})\big],
\label{eq:sampling_mixture_cn}
\end{equation}
where $u(\boldsymbol{\xi})$ denotes the global uniform sampling distribution over the state space, $\epsilon_u$ is a positive constant uniform mixing coefficient, 
$p_g \in (0,1)$ is a constant goal-bias probability, 
$\delta(\boldsymbol{\xi}-\boldsymbol{\xi}_g)$ denotes a point-mass distribution at the goal state $\boldsymbol{\xi}_g$,
and $q(\boldsymbol{\xi})$ denotes the distribution induced by the proposed active sampling strategy.

The sampling procedure is summarized in Algorithm~\ref{alg:bass_srrtstar_cn}.
By introducing Bayesian active subregion sampling, regions with higher posterior means and lower uncertainty are selected more frequently, while regions with limited samples and higher posterior uncertainty still retain a nonzero probability of being explored.
This mechanism effectively reduces invalid sampling, improves the probability of discovering narrow passages, and enhances the overall planning efficiency.
\begin{algorithm}[t]
\caption{Bayesian Active Subregion Sampling Strategy}
\label{alg:bass_srrtstar_cn}
\begin{algorithmic}[1]
\Require Subregion partition $\{\mathcal{X}_k\}_{k=1}^K$, posterior parameters $(\alpha_{k},\beta_{k})$, mixing coefficient $\epsilon_u$
\Ensure Sample point ${\boldsymbol{\xi}}_{rand}$ and posterior update of the selected subregion $\mathcal{X}_k$
\State Sample $u\sim \mathrm{Unif}(0,1)$
\If{$u<\epsilon_u$}
    \State Sample ${\boldsymbol{\xi}}_{rand} \sim \mathrm{Unif}(\mathcal{X})$
    \State $k^{\ast} \gets \mathrm{None}$ 
\Else
    \For{$k=1$ to $K$}
        \State Sample $\hat{\theta}_k \sim \mathrm{Beta}(\alpha_k,\beta_k)$
    \EndFor
    \State Compute $k^{\ast}$ according to (\ref{eq:best_region})
    \State Generate ${\boldsymbol{\xi}}_{rand}$ according to (\ref{eq:sampling})
\EndIf
\State Expand to obtain ${\boldsymbol{\xi}}_{\mathrm{new}}$ according to Algorithm~\ref{alg:adaptive_expand}
\State Collision checking and feasibility test:
      if ${\boldsymbol{\xi}}_{\mathrm{new}}$ can be added to the tree, set $y_t\gets 1$; otherwise $y_t\gets 0$
\If{$ k^{\ast} \neq \mathrm{None}$}
    \State Update $(\alpha_{k^\ast},\beta_{k^\ast})$ according to (\ref{eq:beta_update})
\EndIf
\end{algorithmic}
\end{algorithm}
\subsection{Adaptive Extension Strategy}
Given the nearest node ${\boldsymbol{\xi}}_{\text{near}}\in\mathbb{R}^3$ and the sampled point ${\boldsymbol{\xi}}_{\mathrm{rand}}\in\mathbb{R}^3$, the standard RRT* direction $\boldsymbol{u}_{\mathrm{rand}}\in\mathbb{S}^2$ and the initial constant expansion step size $\epsilon_0$ are defined.
In standard tree expansion, the direction and new node ${\boldsymbol{\xi}}_{\text{new}}\in\mathbb{R}^3$ are computed as
\begin{equation}
\boldsymbol{u}_{\mathrm{rand}}
=
\frac{\boldsymbol{\xi}_{\mathrm{rand}}-\boldsymbol{\xi}_{\mathrm{near}}}
{\|\boldsymbol{\xi}_{\mathrm{rand}}-\boldsymbol{\xi}_{\mathrm{near}}\|},
\label{eq:rand_direction}
\end{equation}
\begin{equation}
\boldsymbol{\xi}_{\mathrm{new}}
=
\boldsymbol{\xi}_{\mathrm{near}}
+
\epsilon_0\,\boldsymbol{u}_{\mathrm{rand}}.
\label{eq:standard_rrt_expand}
\end{equation}

Conventional repulsive fields are typically designed for point obstacles.
However, the volume of a multi-UAV payload transportation system is non-negligible.
In the presence of sharp features (e.g., polyhedral obstacles such as cubes), it is desirable to employ a potential field that both tightly conforms to obstacle geometry and yields smooth transitions.
To this end, a superquadric repulsive potential is adopted, and a goal attractive field is introduced to mitigate goal-unreachability issues.

The potential-field direction $\boldsymbol{u}_F\in\mathbb{S}^2$ is constructed from an obstacle repulsion term $\boldsymbol{F}_{\mathrm{rep}}({\boldsymbol{\xi}})\in\mathbb{R}^3$, a goal attraction term $\boldsymbol{F}_{\mathrm{att}}({\boldsymbol{\xi}},\boldsymbol{\xi}_{g})\in\mathbb{R}^3$, and a linear constraint term $\boldsymbol{F}_{\mathrm{line}}({\boldsymbol{\xi}})\in\mathbb{R}^3$ that regularizes the expansion direction:
\begin{align}
\boldsymbol{u}_F &= \dfrac{\boldsymbol{F}_{\mathrm{tot}}}{\|\boldsymbol{F}_{\mathrm{tot}}\|},
\label{eq:potential_direction}\\
\boldsymbol{F}_{\mathrm{tot}} &=
\boldsymbol{F}_{\mathrm{att}}({\boldsymbol{\xi}},\boldsymbol{\xi}_{g})
+ \boldsymbol{F}_{\mathrm{rep}}({\boldsymbol{\xi}})
+ \boldsymbol{F}_{\mathrm{line}}({\boldsymbol{\xi}}).
\label{eq:total_force}
\end{align}

The obstacle repulsion term is defined as $\boldsymbol{F}_{\mathrm{rep}}= -K_{\mathrm{rep}} \nabla_{{\boldsymbol{\xi}}}U({\boldsymbol{\xi}})$, where $K_{\mathrm{rep}} > 0$ is a constant repulsion gain.
The superquadric potential function is defined as
\begin{align}
U(\boldsymbol{\xi})
=
\frac{A e^{-\eta C(\boldsymbol{\xi})}}{C(\boldsymbol{\xi})},
\end{align}
where $A$ and $\eta$ are positive constants.
The corresponding implicit superquadric function is
\[
C({\boldsymbol{\xi}})
=
\left(
\left(\tfrac{\xi(1)}{f_1({\boldsymbol{\xi}})}\right)^{2b}
+
\left(\tfrac{\xi(2)}{f_2({\boldsymbol{\xi}})}\right)^{2b}
\right)^{\tfrac{2a}{2b}}
+
\left(\tfrac{\xi(3)}{f_3({\boldsymbol{\xi}})}\right)^{2a}
-1,
\]
where $f_i(\boldsymbol{\xi})$ ($i=1,2,3$) are scaling functions, and positive constant $a$ and positive constant $b$ are shape parameters that regulate the geometry of the level sets.

The goal attraction term is defined as $\boldsymbol{F}_{\mathrm{att}}=K_{\mathrm{att}}(\boldsymbol{\xi}_{g}-{\boldsymbol{\xi}})$.
The linear constraint term $\boldsymbol{F}_{\mathrm{line}}$ is introduced to suppress lateral deviations and thus improve local path smoothness:
\begin{align}
\boldsymbol{F}_{\mathrm{line}}
=
- K_{\mathrm{line}}(\boldsymbol{v}-(\boldsymbol{v}^{\top}\boldsymbol{\ell})\,\boldsymbol{\ell}),
\end{align}
where $K_{\mathrm{att}}$ and $K_{\mathrm{line}}$ denote positive constant gains,
$\boldsymbol{v} = \boldsymbol{\xi} - \boldsymbol{\xi_p} \in \mathbb{R}^3$, and $\boldsymbol{\xi_p} \in \mathbb{R}^3$ denotes the parent of the current nearest node $\boldsymbol{\xi}_{\mathrm{near}}$.
The unit vector of the previous expansion direction $\boldsymbol{\ell} \in \mathbb{S}^2$ is defined as
\begin{equation}
\boldsymbol{\ell}
=
\frac{{\boldsymbol{\xi}}_{\mathrm{near}}-{\boldsymbol{\xi}}_{\mathrm{p}}}
{\|{\boldsymbol{\xi}}_{\mathrm{near}}-{\boldsymbol{\xi}}_{\mathrm{p}}\|}.
\label{eq:previous_direction}
\end{equation}

To balance obstacle avoidance safety and sampling-based exploration, an adaptive weighted fusion between the random direction $\boldsymbol{u}_{\mathrm{rand}}$ and the potential-field direction $\boldsymbol{u}_F$ is employed.
Specifically, the fused expansion direction $\tilde{\boldsymbol{d}}$ is defined as
\begin{align}
\tilde{\boldsymbol{d}}
=\frac{\alpha({\boldsymbol{\xi}})\,\boldsymbol{u}_F+\big(1-\alpha({\boldsymbol{\xi}})\big)\,\boldsymbol{u}_{\mathrm{rand}}}
{\left\|\alpha({\boldsymbol{\xi}})\,\boldsymbol{u}_F+\big(1-\alpha({\boldsymbol{\xi}})\big)\,\boldsymbol{u}_{\mathrm{rand}}\right\|},
\label{eq:adaptive_direction}
\end{align}
where $\alpha({\boldsymbol{\xi}})\in[0,1]$ is an adaptive weight defined by
\begin{align}
\alpha({\boldsymbol{\xi}})
&=\alpha_{\mathrm{obs}}\cdot \alpha_{F}\notag\\
&=\left(\frac{d_0-d_{\mathrm{obs}}}{d_0}\right)\left(\frac{\|\boldsymbol{F}_{\mathrm{tot}}\|}{\|\boldsymbol{F}_{\mathrm{tot}}\|+F_0}\right),
\label{eq:alpha}
\end{align}
where $d_{\mathrm{obs}}$ denotes the distance from the current node to the nearest obstacle, $d_0$ and $F_0$ are positive constants, and $d_0$ represents the influence distance threshold of the potential field.
Hence, $\alpha({\boldsymbol{\xi}})$ increases as the node approaches obstacles.

To further enhance robustness, with probability $p_e$ the algorithm falls back to the original random expansion direction, 
preventing the potential field from causing the search to get trapped in local minima in dense obstacles, particularly in narrow passages.
Finally, the expansion direction $\boldsymbol{d}$ is defined as follows:
\begin{equation}
\boldsymbol{d}=
\begin{cases}
\boldsymbol{u}_{\mathrm{rand}}, & p\le p_e,\\
\tilde{\boldsymbol{d}}, & \text{otherwise}.
\end{cases}
\label{eq:fallback_cn}
\end{equation}

Moreover, since the steering capability of the system is limited by a maximum admissible curvature and the associated upper bound on lateral acceleration, the path cannot exhibit arbitrarily large direction changes at corners.
If a fixed step size $\epsilon_0$ is used, invalid and non-executable local segments may occur near sharp turns.
Therefore, the turning angle $\vartheta$ is explicitly incorporated, and a turn-angle adaptive strategy is adopted to shorten the step size and reduce the dynamic burden during turning, thereby improving executability and convergence efficiency.
The angle-related candidate step size is defined as
\begin{align}
\epsilon_1 &= \epsilon_0\Big(\tau+(1-\tau)\exp(-\beta\vartheta/\pi)\Big),
\end{align}
where $\vartheta$ is the angle between the incoming direction and the current expansion direction, $\beta$ is a positive constant turn-angle factor, and $\tau\in(0,1]$ is a constant parameter..

In addition, the step size should adapt to environment complexity \cite{Ning2026TSRIL}.
An environment-dependent candidate step size is defined as
\begin{equation}
\begin{aligned}
\epsilon_2
&=
\epsilon_0\Bigl(1+\Pi_{[-S_{\lambda,\max},\,S_{\lambda,\max}]}\!\bigl(S_{\lambda k}\bigr)\Bigr),\\
\Pi_{[l,u]}(y)
&=
\arg\min_{h\in[l,u]} (h-y)^2,\\
S_{\lambda k}
&=
k_\lambda \frac{\rho_{\mathrm{obs}}(\mathcal{X})-\rho_k^{\mathrm{obs}}}{1-\rho_{\mathrm{obs}}(\mathcal{X})},\\
\rho_{\mathrm{obs}}(\mathcal{X})
&=
\frac{1}{K}\sum_{k=1}^{K}\rho_k^{\mathrm{obs}}.
\end{aligned}
\label{eq:eps_proj_Slambda}
\end{equation}
where $S_{\lambda k}$ is the step-size adjustment factor, $k_{\lambda}>0$ is the density scaling coefficient, $S_{\lambda,\max}>0$ is the saturation bound, 
and $\rho_{\mathrm{obs}}(\mathcal{X})$ denotes the average obstacle density of the configuration space.

To satisfy both obstacle-density effects and dynamic feasibility requirements, the final step size is obtained using a soft-min operator that smoothly selects the smaller value between the two candidates:
\begin{align}
\epsilon
&=\mathrm{softmin}_{\alpha}(\epsilon_1,\epsilon_2)\notag\\
&= \epsilon_m-\frac{1}{\alpha}\log\!\left(
e^{-\alpha(\epsilon_1-\epsilon_m)}+e^{-\alpha(\epsilon_2-\epsilon_m)}
\right),
\end{align}
where $\epsilon_m=\min\{\epsilon_1,\epsilon_2\}$ and $\alpha>0$ is a constant smoothing parameter.

In summary, the improved tree expansion updates the direction and step size according to
\begin{equation}
\boldsymbol{\xi}_{\mathrm{new}}=\boldsymbol{\xi}_{\mathrm{near}}+\epsilon\,\boldsymbol{d}.
\label{eq:tree_expand}
\end{equation}

The overall procedure is summarized in Algorithm~\ref{alg:adaptive_expand}.
\begin{algorithm}[t]
\caption{Adaptive Potential-Guided Tree Expansion Strategy}
\label{alg:adaptive_expand}
\begin{algorithmic}[1]
\Require nearest node ${\boldsymbol{\xi}}_{\mathrm{near}}$, sampled point ${\boldsymbol{\xi}}_{\mathrm{rand}}$, initial step size $\epsilon_0$
\Ensure new node ${\boldsymbol{\xi}}_{\mathrm{new}}$
\State Compute $\boldsymbol{u}_{\mathrm{rand}}$ using~\eqref{eq:rand_direction}
\State Compute $\boldsymbol{u}_F$ using~\eqref{eq:potential_direction}--\eqref{eq:previous_direction}
\State Compute $\tilde{\boldsymbol{d}}$ and the fallback direction using~\eqref{eq:adaptive_direction}--\eqref{eq:fallback_cn}
\State Sample $p\sim \mathrm{Unif}(0,1)$
\If{$p\le p_e$}
    \State $\boldsymbol{d}\gets \boldsymbol{u}_{\mathrm{rand}}$
\Else
    \State $\boldsymbol{d}\gets \tilde{\boldsymbol{d}}$
\EndIf
\State Compute $\epsilon$ using the soft-min strategy
\State Update ${\boldsymbol{\xi}}_{\mathrm{new}}$ using~\eqref{eq:tree_expand}
\State \Return ${\boldsymbol{\xi}}_{\mathrm{new}}$
\end{algorithmic}
\end{algorithm}

\subsection{Composite Cost Function}
\label{sec:composite_cost_cn}
Since the admissible acceleration of the system is limited and the lateral acceleration during turning often becomes the dominant constraint, this paper incorporates a smooth feasibility surrogate for cable-force constraints directly into the online search process, rather than relying solely on post-processing smoothing after obtaining a geometrically feasible path.
To this end, a turning-acceleration energy surrogate is introduced into the cost function to continuously penalize sampled points associated with high curvature and aggressive turning.
As a result, the search is biased toward producing trajectory segments with smoother curvature variations and lower tension requirements.

Consider an edge that expands from a parent node ${\boldsymbol{\xi}}_p$ to a child node ${\boldsymbol{\xi}}_{\mathrm{new}}$.
The corresponding turning acceleration $\boldsymbol{a}\in \mathbb{R}^3$ is computed as
\begin{align}
\boldsymbol{a} 
&= \frac{v_{\mathrm{ref}}\left(\boldsymbol{u}_{\mathrm{out}}-\boldsymbol{u}_{\mathrm{in}}\right)}{\Delta t},
\label{eq:turn_accel2} \\
\boldsymbol{u}_{\mathrm{in}} 
&= \boldsymbol{\xi}_p - \boldsymbol{\xi}_{pp},
\label{eq:uin_def} \\
\boldsymbol{u}_{\mathrm{out}} 
&= \boldsymbol{\xi}_{\mathrm{new}} - \boldsymbol{\xi}_p ,
\label{eq:uout_def}
\end{align}
where $\Delta t$ and $v_{\mathrm{ref}}$ are positive constants denoting the time interval and the reference speed, respectively; 
$\boldsymbol{u}_{\mathrm{out}}, \boldsymbol{u}_{\mathrm{in}} \in \mathbb{S}^2$ denote the outgoing and incoming unit directions, respectively; 
and ${\boldsymbol{\xi}}_{pp}\in \mathbb{R}^3$ denotes the parent of ${\boldsymbol{\xi}}_p$.

Furthermore, the acceleration penalty of the edge is defined as
\begin{equation}
A_e 
=\int \|\boldsymbol{a}(t)\|^2\,dt 
\;\;\approx\;\; \|\boldsymbol{a}\|^2\,\Delta t.
\label{eq:turn_accel_energy}
\end{equation}

This cost directly penalizes the turning acceleration induced by the change in unit direction.
As the turning angle increases, the directional difference
$\|\boldsymbol{u}_{\mathrm{out}}-\boldsymbol{u}_{\mathrm{in}}\|$
also increases, yielding a larger penalty.
Therefore, this term suppresses frequent heading changes during the search stage, effectively reducing ``zigzag'' paths and improving trajectory smoothness.

In summary, the following composite cost function is adopted during parent selection and rewiring:
\begin{equation}
J(\boldsymbol{\xi})=
w_L \sum_{t=1}^{m} \frac{\ell_t}{d_{\mathrm{ref}}}
+
w_A\, A_e,
\label{eq:total_cost_taes}
\end{equation}
where $\ell_t$ denotes the Euclidean distance between two adjacent sampled points, 
$d_{\mathrm{ref}}$ is a normalization constant, and $w_L$ and $w_A$ are constant weighting factors.

Accordingly, when selecting a candidate parent node ${\boldsymbol{\xi}}_{p}$ for the new node ${\boldsymbol{\xi}}_{\mathrm{new}}$, the incremental cost can be computed using only the parent of ${\boldsymbol{\xi}}_{p}$, i.e., ${\boldsymbol{\xi}}_{pp}$, thereby preserving the incremental computability of the overall cost.
During rewiring, only the turning-related term in the locally affected region needs to be updated, which maintains the online efficiency of RRT$^\ast$.

\subsection{Enhanced Tube-RRT$^\ast$ Framework}
Table~\ref{alg:tube_rrtstar_cn} presents the complete workflow of the proposed algorithm and provides the theorems regarding its probabilistic completeness and asymptotic optimality.

\begin{algorithm}[t]
\caption{Enhanced Tube-RRT* Planner}
\label{alg:tube_rrtstar_cn}
\begin{algorithmic}[1]
\Require space $\mathcal{X}$, start state $\boldsymbol{\xi}_{\mathrm{o}}$, goal state $\boldsymbol{\xi}_{\mathrm{g}}$
\Ensure $\mathcal{G}=(V,E)$
\State $V \gets \{\boldsymbol{\xi}_{0}\};\; E \gets \emptyset;\; \mathcal{G}=(V,E)$
\For{$t=1$ \textbf{to} $m$}
\State ${\boldsymbol{\xi}}_{\mathrm{rand}} \gets$ Sampling using Algorithm~\ref{alg:bass_srrtstar_cn}
\State ${\boldsymbol{\xi}}_{\mathrm{new}} \gets$ Tree expansion using Algorithm~\ref{alg:adaptive_expand}
    \If{$r_{\mathrm{new}} > r_{\min}$}
        \State $X_{\mathrm{near}} \gets \mathrm{NearConnect}(\mathcal{G}, \boldsymbol{\xi}_{\mathrm{new}})$
        \State $V \gets V \cup \{\boldsymbol{\xi}_{\mathrm{new}}\}$
        \State $E \gets \mathrm{Rewire}(\boldsymbol{\xi}_{\mathrm{new}}, X_{\mathrm{near}}, E)$
    \EndIf
\EndFor
\State \Return $\mathcal{G}=(V,E)$
\end{algorithmic}
\end{algorithm}

\begin{theorem}[Probabilistic Completeness of the Enhanced Tube-RRT*]
\label{thm:pc_tube_rrtstar_cn}
Let the initial state be ${\boldsymbol{\xi}}_0$ and the goal region ${\boldsymbol{\xi}}_g\subseteq \mathcal{X}_{\mathrm{free}}$.
If the planning problem admits a feasible solution, i.e., there exists at least one collision-free path $\sigma^\star:[0,1]\rightarrow \mathcal{X}_{\mathrm{free}}$ such that
$\sigma^\star(0)={\boldsymbol{\xi}}_0$ and $\sigma^\star(1)={\boldsymbol{\xi}}_g$,
then as the number of samples $m\to\infty$, the probability that the Enhanced \textsc{Tube-RRT*} algorithm finds a feasible solution converges to $1$.
\end{theorem}

\begin{proof}
See Appendix~A.
\end{proof}

\begin{theorem}[Asymptotic Optimality of the Enhanced Tube-RRT*]
\label{thm:pc_tube_rrtstar_ao}
Assume that a feasible solution exists.
As the number of samples $m\to\infty$, the composite cost function obtained by the proposed algorithm $J(\boldsymbol{\xi})$ converges almost surely to the optimal cost $J^\star(\boldsymbol{\xi})$.
\end{theorem}

\begin{proof}
See Appendix~B.
\end{proof}

\textbf{Remark 1}
\emph{Let $m$ denote the number of nodes sampled by the Enhanced Tube-RRT* algorithm.
The computational complexity remains $O(m(\log m + 1))$ \cite{10844529}.
Therefore, the proposed planner can be applied to efficiently generate payload homotopy paths in large-scale environments.
In such cases, the computational complexity becomes
\[
O(m(\log m + 1) + l),
\]
where $l$ denotes the number of homotopy paths.}

\textbf{Remark 2}\emph{
$r_{\mathrm{new}}$ denotes the maximum collision-free radius at the sampled 
configuration, i.e., the distance to the nearest obstacle. 
$r_{\min}$ is the minimum tube radius determined by the physical size of the 
multi-UAV transportation system. 
The condition $r_{\mathrm{new}}>r_{\min}$ in Algorithm~\ref{alg:tube_rrtstar_cn} guarantees that the local free 
space can accommodate the system body without collision.}

\section{Trajectory Optimization}
In this section, the discrete path points generated in the previous section are smoothed.
The payload trajectory in 3-D space is parameterized using piecewise polynomials \cite{783649231210012}.
Meanwhile, the cable force vectors are jointly optimized at discrete sampling points such that the dynamic equilibrium, tension upper bound, and maximum tilt-cone constraints are satisfied.
High-order derivative costs are introduced to ensure trajectory smoothness, while a set of convex regularization terms is employed to obtain smooth, balanced, and feasible cable-force profiles.

First, the payload position trajectory $\boldsymbol{\xi}(t)\in\mathbb{R}^3$ is represented by an $n$-th order piecewise polynomial with $S$ segments:
\begin{equation}
\label{eq:ps_poly_cn}
\boldsymbol{\xi}_s(t)
= \sum_{j=0}^{n} \boldsymbol{c}_{s,j}\,(t-t_{s-1})^{j},
\qquad t\in[t_{s-1},t_s],
\end{equation}
where $\boldsymbol{c}_{s,j}$ denotes the $j$-th polynomial coefficient of the $s$-th segment,and stacking all segment coefficients $\boldsymbol{c}_{s,j}\in\mathbb{R}^3$ yields the overall trajectory parameter vector
$\boldsymbol{c} =
[\boldsymbol{c}_{1,0}^\top,\boldsymbol{c}_{1,1}^\top,\dots,\boldsymbol{c}_{1,n}^\top,
\boldsymbol{c}_{2,0}^\top,\dots,\boldsymbol{c}_{S,n}^\top]^\top$.

The trajectory smoothness objective is defined as
\begin{equation}
\label{eq:Jtraj_cn}
J_{\text{traj}}
=
\frac{1}{2}\int_{t_0}^{t_S}
\left\|
\frac{d^{r}\boldsymbol{\xi}(t)}{dt^{r}}
\right\|^{2}\,dt
=
\frac{1}{2}\boldsymbol{c}^\top \mathbf{Q}\boldsymbol{c},
\end{equation}
where $\mathbf{Q}$ is the corresponding positive semidefinite cost matrix determined by the $r$-th order derivative in the smoothness objective.
Considering the relationship between the system control inputs and higher-order trajectory derivatives, $r\ge5$ is required to ensure smooth control commands.
In addition, the equality constraint $\mathbf{A}_{\mathrm{eq}}\boldsymbol{c}=\boldsymbol{b}_{\mathrm{eq}}$ must be satisfied, corresponding to waypoint and continuity constraints:
\begin{equation}
\label{eq:waypoint_continuity_cn}
\begin{aligned}
& \boldsymbol{\xi}(t_s)=\boldsymbol{w}_s,\qquad s=0,\ldots,S.\\
& \frac{d^q\boldsymbol{\xi}_s(t)}{dt^q}\Big|_{t=t_s}
=
\frac{d^q\boldsymbol{\xi}_{s+1}(t)}{dt^q}\Big|_{t=t_s},\\
& \qquad q=0,\ldots,r-1,\qquad s=1,\ldots,S-1.
\end{aligned}
\end{equation}

Velocity and acceleration bounds are also imposed:
\begin{equation}
\label{eq:va_bounds_cn}
\begin{aligned}
\boldsymbol{v}_{\min} &\le \boldsymbol{v}_k \le \boldsymbol{v}_{\max},\\
\boldsymbol{a}_{\min} &\le \boldsymbol{a}_k \le \boldsymbol{a}_{\max}.
\end{aligned}
\end{equation}

Next, we introduce the decision variable $\boldsymbol{F}_{k,i}\in\mathbb{R}^3$, representing the tension vector transmitted from the $i$-th UAV to the payload at time $t_k$.
To suppress abrupt variations between adjacent sampling instants, both the magnitude and temporal variation of cable tensions are regularized:
\begin{align}
J_{\Delta F} &= \lambda_f\sum_{i=1}^{n_{C}}\sum_{k=1}^{N_s-1}\|\boldsymbol{F}_{k+1,i}-\boldsymbol{F}_{k,i}\|^2, \\
J_{F} &= \lambda_T\sum_{i=1}^{n_{C}}\sum_{k=1}^{N_s}\|\boldsymbol{F}_{k,i}\|^2,
\end{align}
where $N_s$ denotes the number of sampling instants, and $\lambda_f$ and $\lambda_T$ are positive weighting coefficients.

Assuming massless cables and considering only the translational dynamics of the payload, the force equilibrium constraint is
\begin{align}
\sum_{i=1}^{n_{C}}\boldsymbol{F}_{k,i}&=\boldsymbol{r}_k \label{eq:F_ki},\\
\boldsymbol{a}_k &= \mathbf{A}_{k}(t_k)\,\boldsymbol{c}.
\end{align}
Here, $\boldsymbol{r}_k = m_0(\boldsymbol{a}_k + \boldsymbol{g})$ denotes the required resultant force of the payload at time $t_k$, 
and $\mathbf{A}_{k}(t_k)$ denotes the linear mapping from the polynomial coefficient vector $\mathbf{c}$ to the acceleration $\boldsymbol{a}_k$.

An upper bound on the allowable cable tension $T_u$ is imposed:
\begin{equation}
\label{eq:tension_upper_cn}
\|\boldsymbol{F}_{k,i}\|\le T_u .
\end{equation}

Let $\alpha_u$ denote the maximum allowable tilt angle between the cable force direction and the vertical direction $\boldsymbol{e}_3$ in the inertial frame $\mathrm{I}$.
The following constraint is imposed:
\begin{equation}
\label{eq:tilt_cone_cn}
\boldsymbol{e}_3^\top \boldsymbol{F}_{k,i} \ge \cos\alpha_u \|\boldsymbol{F}_{k,i}\|,
\end{equation}
this constraint ensures that $\boldsymbol{F}_{k,i}$ lies within a second-order cone with axis $\boldsymbol{e}_3$ and half-angle $\alpha_u$, while also enforcing a positive vertical force component.

To reduce the risk of cable overload and improve force distribution balance, a tension-sharing regularization term is introduced:
\begin{equation}
\label{eq:J_share_cn}
J_{\text{share}}=
\lambda_{\text{share}}
\sum_{i=1}^{n_{C}}\sum_{k=1}^{N_s}
\left\|
\boldsymbol{F}_{k,i}-\frac{\boldsymbol{r}_k}{n_{C}}
\right\|^2 ,
\end{equation}
where $\lambda_{\mathrm{share}}$ is a positive weighting coefficient.

To account for payload attitude variations, a moment-tracking cost term is further introduced:
\begin{equation}
\label{eq:J_moment_cn}
J_{\text{moment}}
=
\lambda_M
\sum_{k=1}^{N_s}
\left\|
\boldsymbol{M}_k-\boldsymbol{M}^{\mathrm{ref}}_{d,k}
\right\|^2 ,
\end{equation}

\begin{equation}
\label{eq:M_k_def}
\boldsymbol{M}_k
=
\sum_{i=1}^{n_{C}}
\widehat{\boldsymbol{\rho}}_{i}
\Big(
(\boldsymbol{R}^{\mathrm{ref}}_{0d,k})^{\!\top}
\boldsymbol{F}_{k,i}
\Big),
\end{equation}

\begin{equation}
\label{eq:Md_ref_def}
\boldsymbol{M}^{\mathrm{ref}}_{d,k}
=
\boldsymbol{J}_0\,\dot{\boldsymbol{\Omega}}_{0d,k}
+
\widehat{\boldsymbol{\Omega}_{0d,k}}\boldsymbol{J}_0\,\boldsymbol{\Omega}_{0d,k} ,
\end{equation}
where $\lambda_{M}$ is a positive weighting coefficient;
The vector $\boldsymbol{M}_k\in\mathbb{R}^3$ represents the total moment acting on the payload generated by the distributed UAV forces at the $k$th prediction step,
The matrix $\boldsymbol{R}^{\mathrm{ref}}_{0d,k}\in\mathrm{SO}(3)$ represents the reference attitude of the payload at step $k$, whose computation follows the method proposed in \cite{7843619}. 
The desired moment $\boldsymbol{M}^{\mathrm{ref}}_{d,k}\in\mathbb{R}^3$ is determined from the reference rotational dynamics of the payload, where $\boldsymbol{\Omega}_{0d,k}\in\mathbb{R}^3$ is the desired payload angular velocity, and $\dot{\boldsymbol{\Omega}}_{0d,k}\in\mathbb{R}^3$ is the corresponding desired angular acceleration at step $k$. 

Based on the above considerations, we formulate the following convex quadratic program:
\begin{equation}
\label{eq:final_problem_cn}
\begin{aligned}
\min_{\boldsymbol{c},\boldsymbol{F}_{k,i}} \quad
& J_{\text{traj}} + J_{\Delta F} + J_{F} + J_{\text{share}} + J_{\text{moment}} \\
\text{s.t.}\quad
& \boldsymbol{A}_{\mathrm{eq}}\boldsymbol{c}=\boldsymbol{b}_{\mathrm{eq}},\\
& \boldsymbol{v}_{\min}\le \boldsymbol{A}^{(1)}_k\boldsymbol{c}\le \boldsymbol{v}_{\max},\\
& \boldsymbol{a}_{\min}\le \boldsymbol{A}^{(2)}_k\boldsymbol{c}\le \boldsymbol{a}_{\max},\\
& \sum_{i=1}^{n_{C}}\boldsymbol{F}_{k,i}=m_0(\boldsymbol{A}^{(2)}_k\boldsymbol{c}+\boldsymbol{g}),\\
& \|\boldsymbol{F}_{k,i}\| \le T_u,\\
& \boldsymbol{e}_3^\top \boldsymbol{F}_{k,i} \ge \cos\alpha_u\,\|\boldsymbol{F}_{k,i}\|.
\end{aligned}
\end{equation}

Problem~\eqref{eq:final_problem_cn} consists of a convex quadratic objective, linear equality/inequality constraints, and second-order cone (SOC) constraints, and can thus be solved efficiently using off-the-shelf convex optimization solvers.
In this work, we use \texttt{CVX} as the modeling interface and \texttt{MOSEK} as the underlying solver.

The feasibility of \eqref{eq:final_problem_cn} is primarily determined by the force equilibrium constraint and the feasible set of cable tensions.
A feasible initialization is constructed as follows.
First, we ignore the cable constraints and solve for the minimum high-order derivative trajectory to obtain $r_k^{(0)}$, then an initial tension allocation is set as
\begin{equation}
\label{eq:F0_guess_cn}
\boldsymbol{F}_{k,i}^{(0)} \approx \frac{r_k^{(0)}}{n_{C}}.
\end{equation}

The direction of $\boldsymbol{F}_{k,i}^{(0)} $ is then projected onto the tilt-cone constraint to ensure consistency with the angle bound.
This initialization is also aligned with the tension-sharing soft constraint in the objective, which reduces the probability of overloading a single cable and improves allocation balance when the constraints are satisfied.

Next, note from \eqref{eq:Md_ref_def} that if $\boldsymbol{R}^{\mathrm{ref}}_{0d,k}$ is determined online from the current trajectory, a nonconvex coupling would be introduced.
Therefore, we adopt a \emph{fixed-reference} strategy: $\boldsymbol{R}^{\mathrm{ref}}_{0,k}$ and $\boldsymbol{M}^{\mathrm{ref}}_{d,k}$ are computed from the initial trajectory $\boldsymbol{r}^{(0)}_k$, and the moment-tracking term is constructed as a convex quadratic function of $\boldsymbol{F}_{k,i}$, thereby preserving convexity in each solve.
When the discrepancy between the reference and the final trajectory becomes large, consistency can be improved by adjusting the weight $\lambda_M$ and/or updating the reference moment $\boldsymbol{M}^{\mathrm{ref}}_{d,k}$.
With the above procedure, the feasible set is nonempty; since the constraint set is closed and the objective is lower bounded, an optimal solution exists.

\textbf{Remark 3}
\emph{
Different from \cite{Chai2025TMechTrajPlan}, this work transforms an offline tension-allocation problem into an online trajectory optimization problem that explicitly enforces tension feasibility.
It is important to emphasize that, in a two-stage pipeline, if the smoothing stage does not explicitly account for the tension magnitude bound $T_u$ and the swing/tilt angle bound $\alpha_u$, the subsequent tension allocation can easily become infeasible---especially when $T_u$ and $\alpha_u$ are small or when large trajectory accelerations increase the required resultant force.
In such cases, there may not exist a set of tensions $\{\boldsymbol{F}_{k,i}\}$ satisfying \eqref{eq:F_ki}--\eqref{eq:tilt_cone_cn} simultaneously.
In contrast, by jointly optimizing the trajectory parameter vector $\boldsymbol{c}$ and the tension variables $\boldsymbol{F}_{k,i}$ during planning, 
the proposed method maintains a higher level of feasibility and executability even under tight tension constraints.}
\emph{
Moreover, the number of decision variables $n_x$ in \eqref{eq:final_problem_cn} is
\begin{equation}
\begin{aligned}
n_x &= \dim(\boldsymbol{c})+\dim(\boldsymbol{F})\\
&=3S(n+1) + 3N_s n_{C} .
\end{aligned}
\end{equation}
When the polynomial order $n$, the number of cables $n_{C}$, and the number of sampling instants $N_s$ are fixed, the problem size grows linearly with the number of segments $S$, 
which is well suited for repeated solves and online replanning in large-scale scenarios.}
\section{Trajectory Generation and Tracking}
To validate the effectiveness and feasibility of the proposed two-stage trajectory planning framework, 
this section compares the trajectory generation results of our method against competing baselines 
and evaluates trajectory tracking using a centralized geometric control strategy\cite{Gao2025RobustnessEF}.

\subsection{Trajectory Generation Results and Analysis}
We compare the proposed Enhanced Tube-RRT$^\ast$ with a mixed-sampling Tube-RRT$^\ast$ (STube-RRT$^\ast$) and an adaptive-extension Tube-RRT$^\ast$ (AETube-RRT$^\ast$).
The test environment has dimensions $50~\mathrm{m}\times 50~\mathrm{m}\times 15~\mathrm{m}$, 
with the start position $\boldsymbol{\xi}_{\mathrm{o}}$ at $(6,\,2,\,10)\,\mathrm{m}$ and the goal position $\boldsymbol{\xi}_{\mathrm{g}}$ at $(40,\,40,\,5)\,\mathrm{m}$.

Under the same iteration budget, ablation experiments were conducted using 50 random seeds.
Table~\ref{tab:param_setting} lists the key parameter settings used in the trajectory generation stage.

Fig.~\ref{fig:search_process_s1} illustrates a representative search process selected from the 50 trials (filtered based on the turning-angle sum and path length).
The thick red curve represents the final path, while the blue circles denote the safety radius.

\begin{table}[t]
  \label{tab:param_setting}
\centering
\small
\setlength{\tabcolsep}{4pt} 
\begin{tabularx}{\columnwidth}{l c Y}
\hline
Symbol & Value & Description \\
\hline
$\epsilon_u$ & 0.2 & Uniform mixture coefficient \\
$p_g$ & 0.2 & Goal-biased sampling probability \\
$\kappa_p$ & 10  & Prior strength of the controller guidance \\
$\epsilon_0$ & 3 & Local extension step size \\
$\beta$ & 2.5 & Turning-angle penalty weight \\
$\alpha$ & 8 & Soft-min shaping parameter \\
$d_0$ & 3 & Distance threshold at which the potential field becomes active \\
$p_e$ & 0.2 & Random expansion probability \\
$v_{\mathrm{ref}}$ & 1.0 & Reference speed (m/s) \\
$\Delta t$ & 0.01 & Minimum integration step size \\
$w_A$ & 1 & Acceleration cost weight \\
$w_L$ & 1 & Distance cost weight \\
$d_{\mathrm{ref}}$ & 2.3 & Scale-normalization constant (to be specified) \\
\hline
\end{tabularx}
\end{table}

Table~\ref{tab:4} compares different algorithms using four metrics, including the number of collision checks, the time to the first feasible solution $T_{\text{first}}$, the total turning angle, and the path length.
The results indicate that introducing the active sampling strategy helps guide the search toward the region containing feasible solutions more quickly, thereby reducing $T_{\text{first}}$ and improving the overall computational efficiency.
However, without the potential-field term to regulate local geometry and smoothness, the generated paths tend to exhibit strong zigzag behavior, reflected by larger turning-angle sums.

In contrast, incorporating the repulsive potential field significantly reduces the number of obstacle collision checks, preserves effective nodes, and suppresses extreme turning angles.
Nevertheless, due to the influence of potential-field guidance and step-size filtering, the search process becomes more conservative, resulting in a relatively longer search time.
Overall, the Enhanced Tube-RRT* produces paths with the smallest turning-angle sum and path length, achieving the highest path quality among the compared methods.

Table~\ref{tab:results_cn} further summarizes the overall performance across all trials.
The proposed method achieves a significantly higher success rate than both STube-RRT* and AETube-RRT*.
In terms of the effective sampling ratio, the proposed method significantly outperforms STube-RRT* and is comparable to AETube-RRT*.
Moreover, the proposed algorithm finds the first feasible solution faster, indicating that the adaptive expansion strategy increases the probability of valid expansions.
However, because the search tends to bias toward the interior of free space, the first feasible solution may occasionally appear slightly later.

To further evaluate solution quality, we analyze only the successful trials.
The proposed method yields a significantly shorter average path length than STube-RRT*, while the difference with AETube-RRT* is not statistically significant.
Regarding smoothness, the average turning-angle sum of the proposed method is lower than those of both baseline algorithms, indicating smoother trajectories.

To avoid selection bias caused by evaluating only successful samples, we introduce a penalized metric to incorporate failed trials into the overall utility evaluation.
For a metric $m_s$, we define
\begin{equation}
\tilde m_s =
\begin{cases}
m_s, & \text{if successful},\\
1.1 \times \max\limits_{s' \in \mathcal{S}_{\text{succ}}} m_{s'}, & \text{otherwise}.
\end{cases}
\end{equation}
where $\mathcal{S}_{\text{succ}}$ denotes the set of successful trials.
This formulation ensures that any failed run is strictly worse than all successful runs in the ranking.

The results show that, since both baseline algorithms fail in more than half of the random seeds, the median of their penalized metrics degenerates to the penalty values (path lengths of $10.25$ and $9.53$, and turning-angle sums of $25.34$ and $23.43$, respectively).
In contrast, the median values of the proposed method are still dominated by successful samples (path length $7.69$ and turning-angle sum $16.60$).

Overall, the Enhanced Tube-RRT* demonstrates clear advantages in path quality.
Under a comparable computational budget, the proposed method achieves shorter paths, higher success rates, and higher effective sampling ratios, while exhibiting smaller turning-angle sums and smoother trajectories.
These results verify that the proposed planner achieves an excellent balance between robustness and solution quality in dense obstacle environments.
\begin{figure*}[t]
\centering
\subfloat[Enhanced Tube-RRT$^\ast$]{\includegraphics[width=0.33\linewidth]{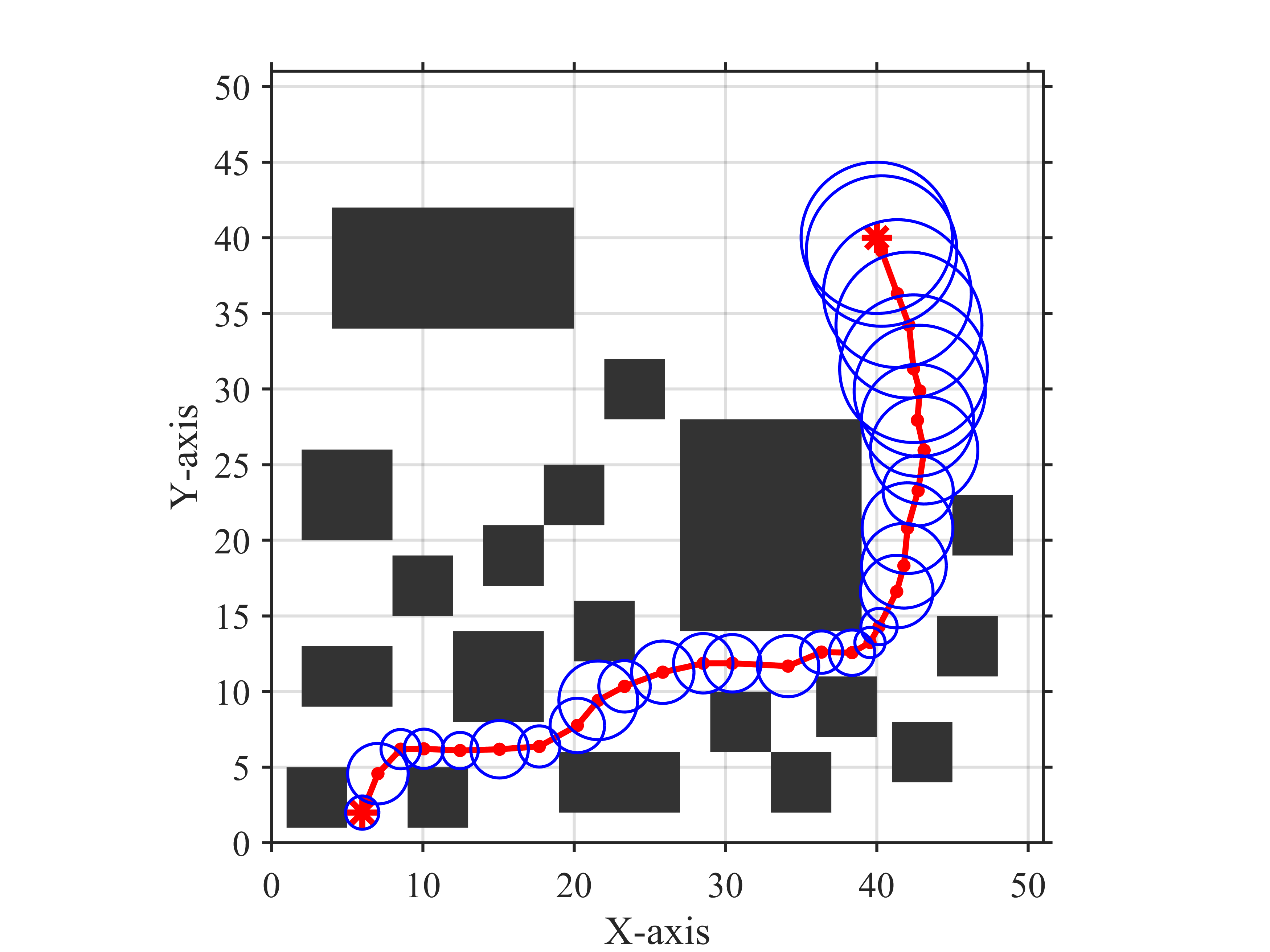}}
\subfloat[STube-RRT$^\ast$]{\includegraphics[width=0.33\linewidth]{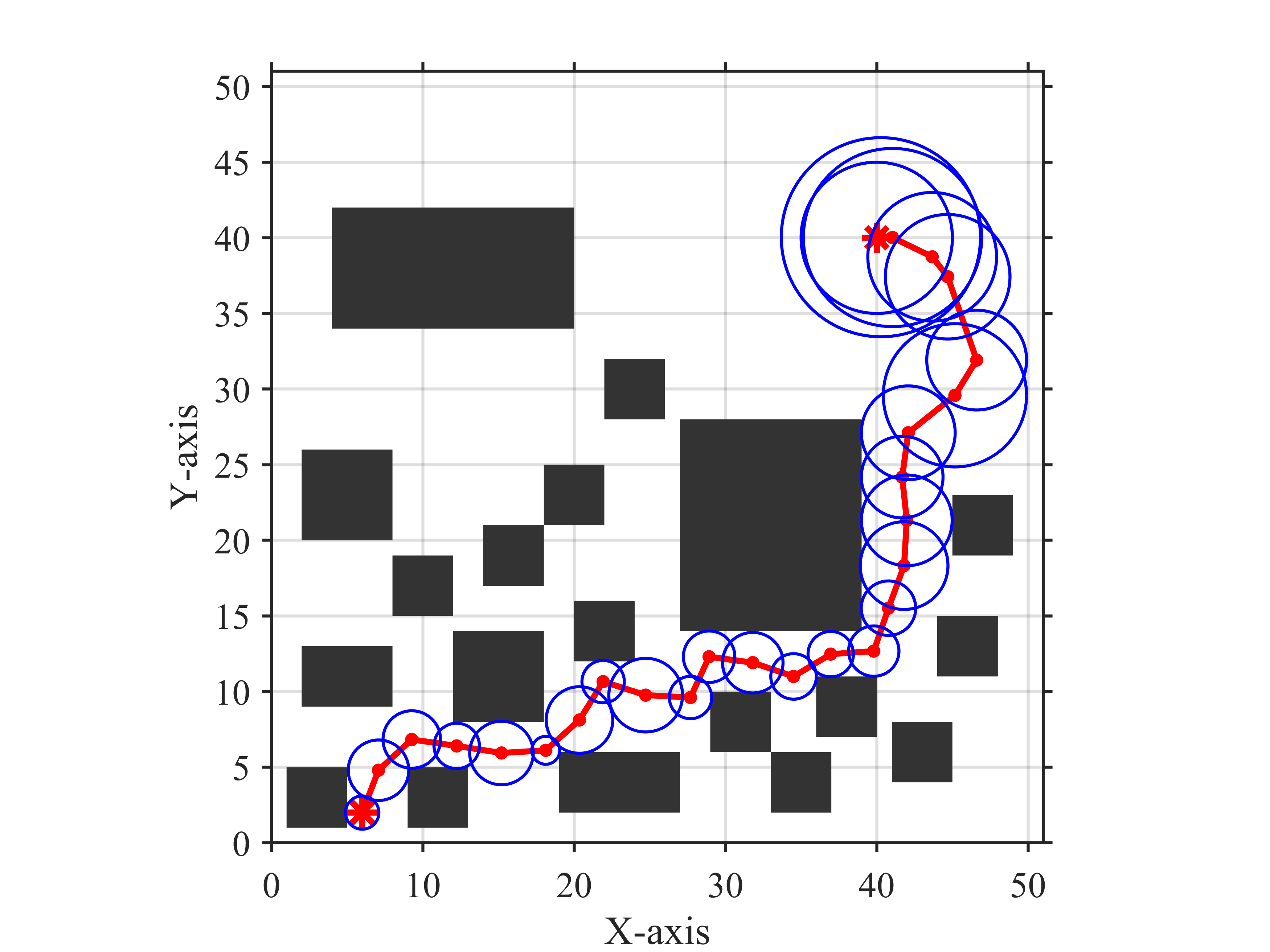}}
\subfloat[AETube-RRT$^\ast$]{\includegraphics[width=0.33\linewidth]{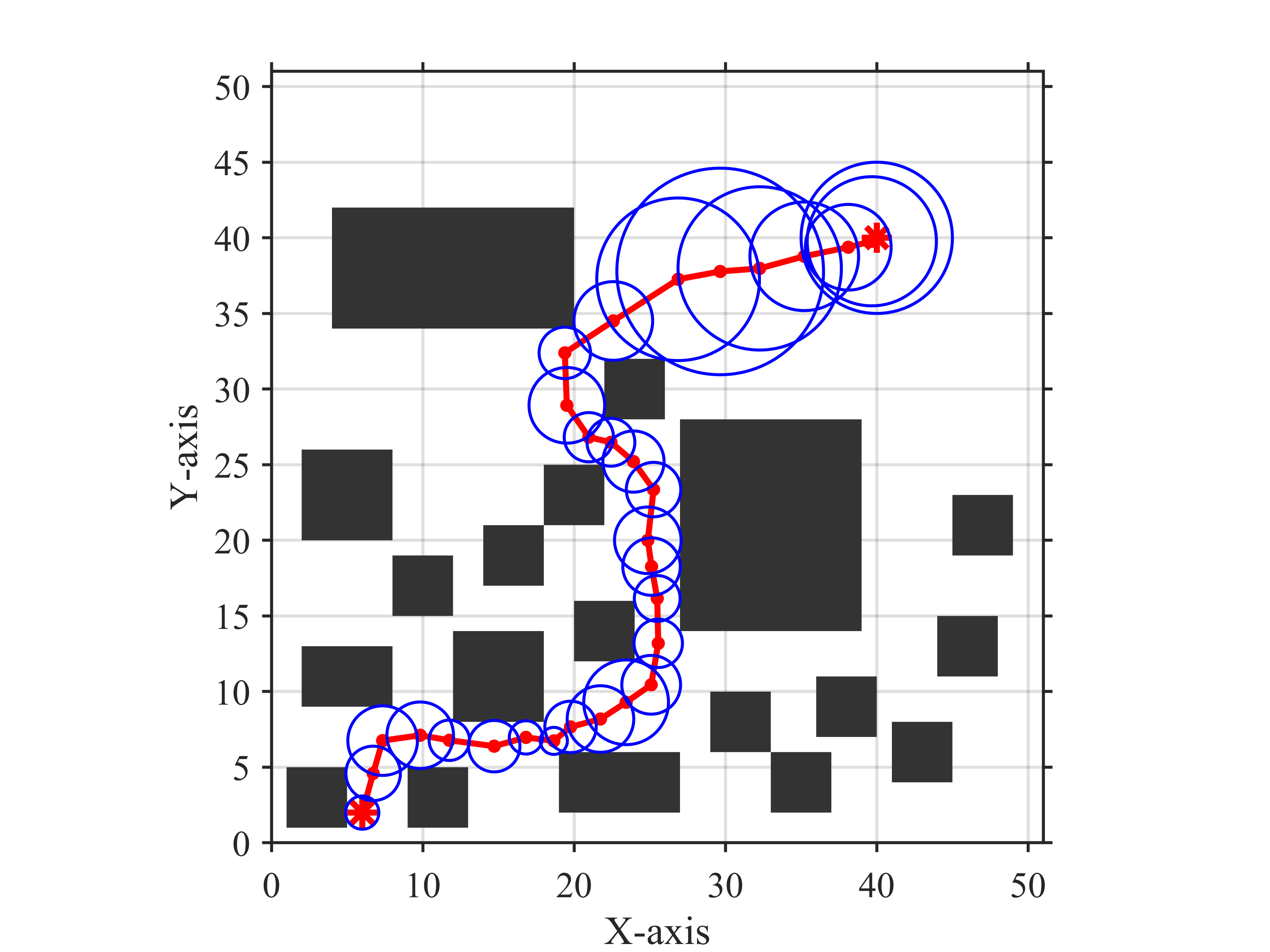}}
\caption{Comparison of paths generated by different algorithms. The thick red curve denotes the final path and the blue circles indicate the corresponding safety radii.
(a) shows the search result of the proposed Enhanced Tube-RRT$^\ast$, which generates a smooth trajectory with relatively uniform safety radii along the path.
(b) presents the path obtained by STube-RRT$^\ast$, where the mixed-sampling strategy accelerates the discovery of feasible regions but leads to noticeable zigzag behavior due to the lack of local geometric regulation.
(c) illustrates the result of AETube-RRT$^\ast$, where the adaptive extension guided by the potential field produces relatively smoother segments but may introduce larger safety radii and longer exploration near obstacles.}
\label{fig:search_process_s1}
\end{figure*}

\begin{table}[t]
\caption{Performance comparison in a representative search process.}
\label{tab:4}
\centering
\begin{tabular}{lccc}
\toprule
Metric & STube-RRT$^\ast$ & AETube-RRT$^\ast$ & Ours \\
\midrule
Collision checks & 3656 & 1752 & 1278 \\
$T_{\text{first}}$ (s)  & 2.66 & 13.88 & 1.13 \\
Turning-angle sum (rad) & 14.56 & 15.08 & 10.97 \\
Path length (m) & 7.75 & 7.57 & 7.27 \\
\bottomrule
\end{tabular}
\end{table}

\begin{table}[t]
\caption{Overall performance comparison across 50 random seeds.}
\label{tab:results_cn}
\centering
\footnotesize
\setlength{\tabcolsep}{4pt}
\renewcommand{\arraystretch}{0.9}
\begin{tabular}{lccc}
\toprule
Metric & STube-RRT$^\ast$ & AETube-RRT$^\ast$ & Ours \\
\midrule
Success rate (\%) & 36 & 28 & 94 \\
Effective sampling ratio (\%) & 39.08 & 52.17 & 51.83 \\
$T_{\text{first}}$ (s) & 1.55 & 1.76 & 1.11 \\
Turning-angle sum (rad) & 18.80 & 17.27 & 17.00 \\
Path length (m) & 8.37 & 7.87 & 7.72 \\
\bottomrule
\end{tabular}
\end{table}
\subsection{Trajectory Tracking Results}
In this subsection, the reference (optimal) trajectory is generated by the proposed two-stage planning pipeline described above, 
and the centralized geometric control scheme follows~\cite{Gao2025RobustnessEF}.
The system parameters and trajectory-optimization weights are summarized in Table~\ref{tab:param_cn}.

\begin{table}[t]
\caption{System parameters and trajectory-optimization weights.}
\label{tab:param_cn}
\centering
\small
\setlength{\tabcolsep}{4pt} 
\renewcommand{\arraystretch}{1.15} 
\begin{tabularx}{\columnwidth}{l c Y}
\hline
Symbol & Value & Description \\
\hline
$m_0~(\mathrm{kg})$ & 5 & Payload mass \\
$\mathbf{J}_0~(\mathrm{kg\,m^2})$ &
\makecell[c]{$\mathrm{diag}(0.6875,$\\$0.59375,$\\$0.78333)$} &
Payload inertia matrix \\
$m_i~(\mathrm{kg})$ & 1 & Mass of each quadrotor \\
$\mathbf{J}_1~(\mathrm{kg\,m^2})$ &
\makecell[c]{$\mathrm{diag}(0.065,$\\$0.030,$\\$0.100)$} &
Inertia of quadrotor~1 \\
$\mathbf{J}_2~(\mathrm{kg\,m^2})$ &
\makecell[c]{$\mathrm{diag}(0.052,$\\$0.035,$\\$0.095)$} &
Inertia of quadrotor~2 \\
$\mathbf{J}_3~(\mathrm{kg\,m^2})$ &
\makecell[c]{$\mathrm{diag}(0.045,$\\$0.048,$\\$0.075)$} &
Inertia of quadrotor~3 \\
$l_i~(\mathrm{m})$ & 0.75 & Cable length \\
$\lambda_f~$ & 10 & Weight on force-rate variation \\
$\lambda_T~$ & 1 & Weight on force magnitude \\
$\lambda_{\mathrm{share}}$& 1 & Weight on force distribution balance \\
$\lambda_M~$ & 0.1 & Weight on moment term \\
\hline
\end{tabularx}
\end{table}

\begin{figure}[htbp]
\centering
\begin{minipage}{0.48\linewidth}
\centering
\includegraphics[width=\linewidth]{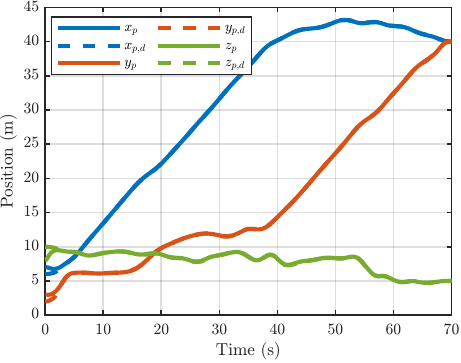}
\\ {\footnotesize (a) Load position tracking curve.}
\end{minipage}
\hfill
\begin{minipage}{0.48\linewidth}
\centering
\includegraphics[width=\linewidth]{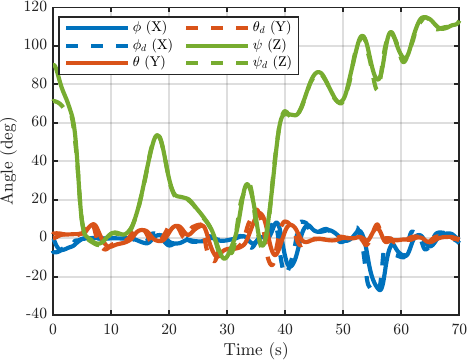}
\\ {\footnotesize (b) Load attitude tracking curve.}
\end{minipage}
\caption{Payload position and attitude tracking performance. 
The payload trajectory closely follows the reference trajectory throughout the mission. 
Although small deviations appear in the roll $\phi$ and pitch $\theta$ angles during aggressive turning, they rapidly converge to the desired values, demonstrating stable tracking performance of the proposed controller.}
\label{fig:PayloadTracking}
\end{figure}

\begin{figure*}[t]
\centering
\begin{minipage}[t]{0.48\textwidth}
\centering
\includegraphics[width=\linewidth,height=5cm,keepaspectratio]{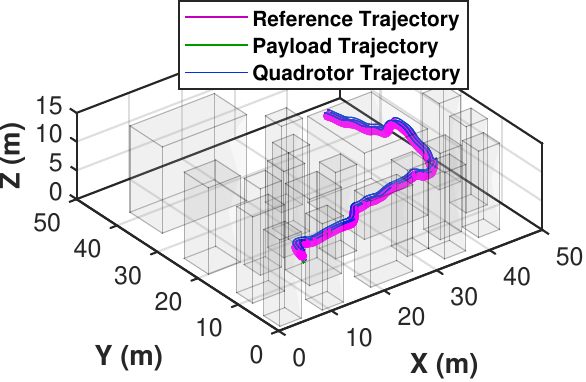}
\\{\footnotesize (a) Side view.}
\end{minipage}
\hfill
\begin{minipage}[t]{0.48\textwidth}
\centering
\includegraphics[width=\linewidth,height=5cm,keepaspectratio]{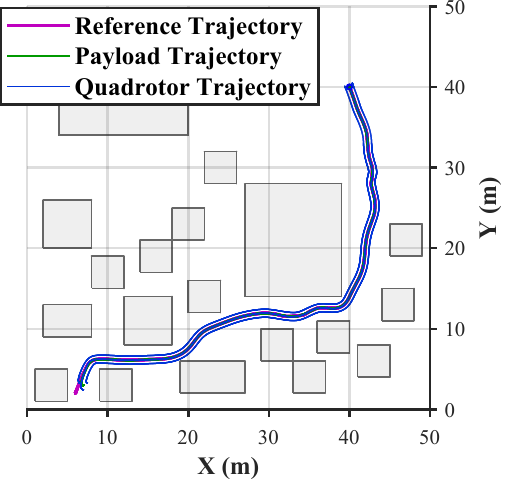}
\\{\footnotesize (b) Top view.}
\end{minipage}
\caption{Spatial trajectories of the UAV–payload cooperative transportation system. 
The gray cubes represent environmental obstacles. 
The trajectories remain collision-free with sufficient safety clearance, demonstrating the effectiveness of the proposed planning and control strategy.}
\label{fig:views}

\end{figure*}

The initial payload position is set to $\boldsymbol{\xi}_0=(7,\,3,\,8)\,\mathrm{m}$,
while the target position is identical to that specified in the previous subsection.
The cable tension bounds are defined as $T_l=1~\mathrm{N}$ and $T_u=35~\mathrm{N}$.
The angle between the cable direction and the vertical axis $\boldsymbol{e}_3$ of the inertial frame $\mathrm{I}$ is constrained by
$\alpha_l=0^{\circ}$ and $\alpha_u=22^{\circ}$.
The velocity and acceleration limits of the payload are set to
$v_{\max}=2~\mathrm{m/s}$ and $a_{\max}=1~\mathrm{m/s^2}$, respectively.

The payload position tracking results are shown in Fig.~\ref{fig:PayloadTracking}(a).
The dashed curve represents the reference trajectory generated by the proposed planner,
while the solid curve denotes the trajectory tracked by the controller.
The actual trajectory closely follows the reference trajectory throughout the entire task.
Stable tracking is maintained during acceleration, turning, and altitude variations,
indicating that the proposed controller can effectively track the planned trajectory while satisfying the feasibility requirements for smooth payload transportation.
\begin{figure*}[!t]
    \centering
    \includegraphics[width=0.8\textwidth,height=0.8\textheight,keepaspectratio]{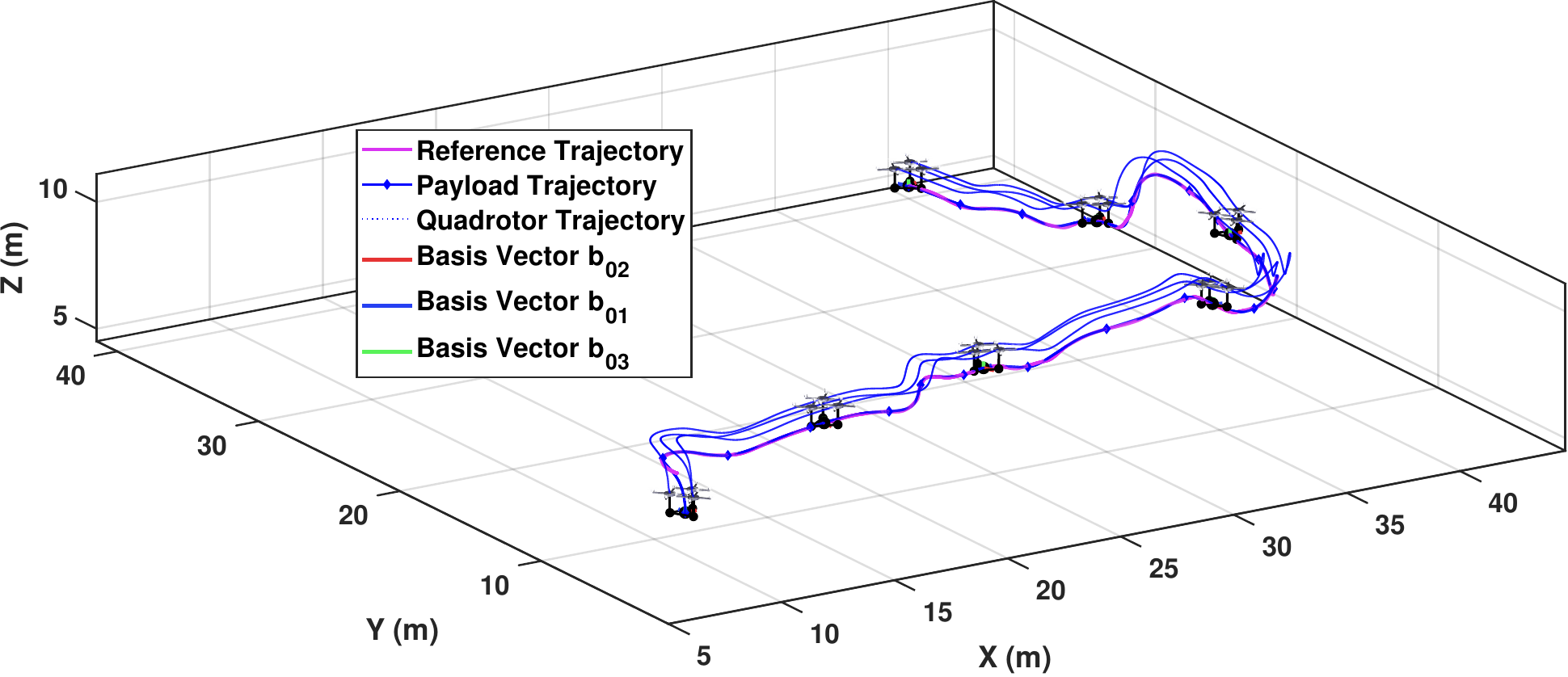}
    \caption{Visualization of the simulated trajectory tracking performance of the proposed trajectory planning and cascaded control framework. 
    The reference trajectory, payload trajectory, and quadrotor trajectory are illustrated together with the payload body-frame basis vectors. 
    The results show that the payload closely follows the reference path while the quadrotor maintains a smooth motion profile. 
    The oscillations of the suspended payload are effectively suppressed throughout the maneuver, demonstrating the effectiveness of the proposed method in achieving accurate trajectory tracking and stable transportation of suspended loads.}
    \label{fig:time_evolution}
\end{figure*}

\begin{figure*}[t]
\centering
\begin{minipage}[t]{0.49\textwidth}\centering
  {\bfseries a}\par\vspace{2pt}
  \includegraphics[width=\textwidth]{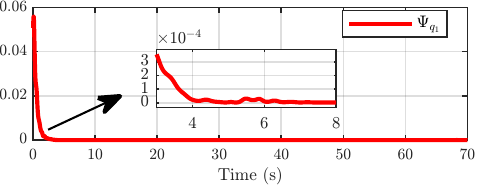}\par\vspace{4pt}
  \includegraphics[width=\textwidth]{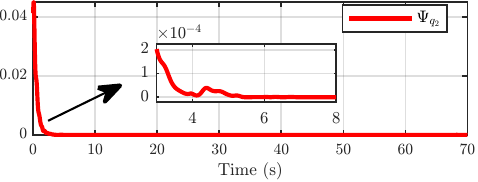}\par\vspace{4pt}
  \includegraphics[width=\textwidth]{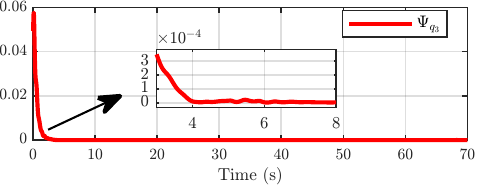}
\end{minipage}\hfill
\begin{minipage}[t]{0.49\textwidth}\centering
  {\bfseries b}\par\vspace{2pt}
  \includegraphics[width=\textwidth]{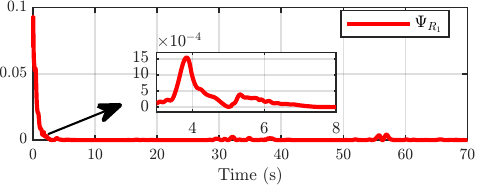}\par\vspace{4pt}
  \includegraphics[width=\textwidth]{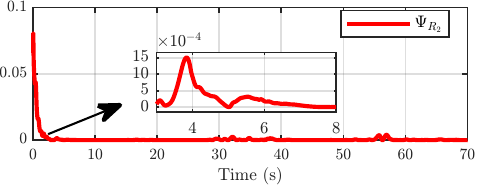}\par\vspace{4pt}
  \includegraphics[width=\textwidth]{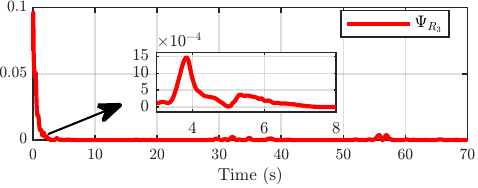}
\end{minipage}
\caption{Configuration errors during trajectory tracking. 
(a) Cable direction configuration errors. 
(b) UAV attitude configuration errors. 
Both errors converge to the order of $10^{-4}$ within approximately $3\,\mathrm{s}$, 
and only small bounded oscillations remain in steady state, demonstrating geometric consistency and effective suppression of coupled oscillations.}
\label{fig:search_process_s2}
\end{figure*}

\begin{figure*}[t]
\centering
\begin{minipage}[t]{0.49\textwidth}\centering
  {\bfseries a}\par\vspace{2pt}
  \includegraphics[width=\textwidth]{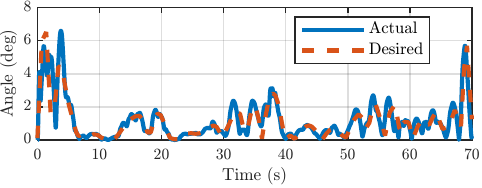}\par\vspace{4pt}
  \includegraphics[width=\textwidth]{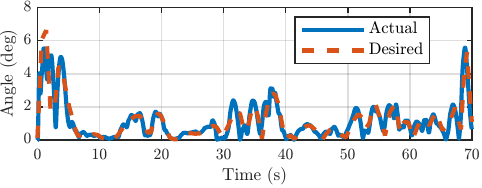}\par\vspace{4pt}
  \includegraphics[width=\textwidth]{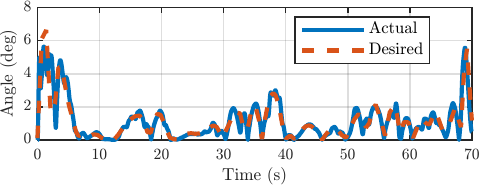}
\end{minipage}\hfill
\begin{minipage}[t]{0.49\textwidth}\centering
  {\bfseries b}\par\vspace{2pt}
  \includegraphics[width=\textwidth]{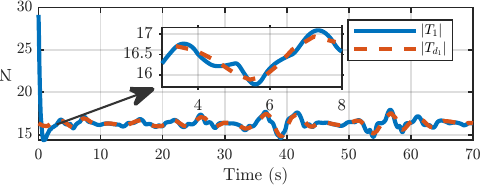}\par\vspace{4pt}
  \includegraphics[width=\textwidth]{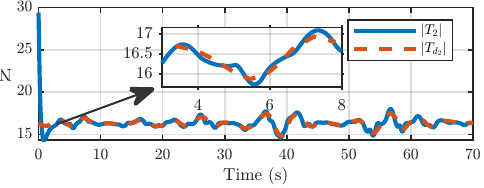}\par\vspace{4pt}
  \includegraphics[width=\textwidth]{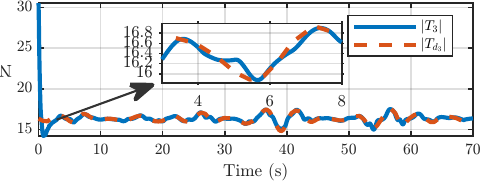}
\end{minipage}
\caption{Rope sway angle and tension amplitude variation curve during trajectory optimization and tracking.
(a) Rope swing angle variations of the three cables. 
(b) Cable tension magnitudes. 
The swing angles remain within prescribed bounds, while the tensions vary smoothly without slack, 
indicating stable force distribution and reduced payload oscillation.}
\label{fig:search_process_s3}
\end{figure*} 

Fig.~\ref{fig:views} illustrates the spatial motion of the UAV–payload cooperative transportation system.
The gray cubes represent environmental obstacles.
From both the top and side views, it can be observed that the UAV–payload system maintains sufficient safety clearance from obstacles throughout the mission.
No obvious trajectory segments approach or penetrate obstacle boundaries,
demonstrating that the proposed planning and control strategy can effectively satisfy spatial constraints while ensuring trajectory feasibility.

The payload attitude tracking results are presented in Fig.~\ref{fig:PayloadTracking}(b).
The roll angle $\phi$ and pitch angle $\theta$ exhibit slightly increased deviations during aggressive turning segments but quickly return to the desired values.

The configuration errors of the cable direction and UAV attitude are shown in Fig.~\ref{fig:search_process_s2}.
Both errors converge to the order of $10^{-4}$ within approximately $3\,\mathrm{s}$.
The zoomed-in view indicates that only small bounded oscillations remain during the steady-state phase.
These results demonstrate that the cable direction and UAV attitude maintain geometric consistency during trajectory tracking.
The corresponding visualization is provided in Fig.~\ref{fig:time_evolution}, which further confirms that the cascaded control structure effectively suppresses coupled oscillations of the suspended system and improves payload trajectory tracking performance.

Fig.~\ref{fig:search_process_s3} further presents the variations of cable swing angles and tension magnitudes during both the trajectory optimization and tracking processes.
The cable swing angles remain within the prescribed constraints throughout the transportation task.
Meanwhile, the cable tensions exhibit smooth variations without abrupt changes or prolonged slack conditions.
This indicates that the system maintains stable force distribution during trajectory tracking, thereby reducing payload swing and improving transportation safety.

These results verify that the proposed trajectory planning and
cascaded control framework enables safe and dynamically feasible
aerial transportation in cluttered environments.
\section{Conclusion}
This paper addresses trajectory planning for a cascaded aerial transportation system and proposes an Enhanced Tube-RRT$^\ast$ planner that incorporates an active mixed-sampling mechanism in the sampling stage and an adaptive expansion strategy in the extension stage, thereby jointly promoting collision-avoidance safety and trajectory smoothness.
Building on the planned path, we further develop a force-feasibility-aware trajectory optimization method.
Finally, closed-loop tracking of the resulting trajectory is validated using a geometric controller.
Simulation results demonstrate the feasibility and effectiveness of the proposed planning-and-optimization framework.
Future work will focus on obstacle-avoidance trajectory generation for cascaded aerial transportation systems in dynamic environments, and on extending the framework to cooperative missions and joint planning for multi-cascaded transportation systems.
\section*{APPENDIX}
\appendices
\section{Proof of Theorem~\ref{thm:pc_tube_rrtstar_cn}}
\label{app:proof_thm1}
Let $n \triangleq |V|$ denote the number of nodes in the current tree,
and let $\mu(\mathcal{X}_{\mathrm{free}})$ denote the volume of the free space.
To ensure sufficient local connectivity and rewiring opportunities during early iterations,
a density-corrected neighborhood radius is adopted:
\begin{equation}
r_n=\min\left\{\left(\frac{\gamma\log n}{\rho\,n}\right)^{1/3},\, 3\epsilon\right\}.
\label{eq:rn_cn}
\end{equation}
where $\gamma>0$ is a constant and $\rho$ represents a lower bound on the effective sampling density over the free space.

Since the sampling distribution follows a mixture strategy that includes a uniform sampling component with weight $\epsilon_u$,
for any measurable subset $A\subseteq \mathcal{X}_{\mathrm{free}}$ the uniform component guarantees
\begin{equation}
\mathbb{P}(\boldsymbol{\xi}_{\mathrm{rand}}\in A)
\ \ge\
\epsilon_u\,\frac{\mu(A)}{\mu(\mathcal{X}_{\mathrm{free}})}.
\end{equation}

Furthermore, to avoid repeated failures caused by turning-angle constraints or potential-field guidance in the adaptive expansion strategy,
we assume the existence of a minimum expansion success probability $p_l\in(0,1]$.
That is, conditioned on $\boldsymbol{\xi}_{\mathrm{rand}}\in \mathcal{X}_{\mathrm{free}}$,
the probability of generating a valid new node is at least $p_l$.

Consequently, the probability that a valid new node is generated within region $A$ in one iteration admits the conservative lower bound
\begin{equation}
\label{eq:rho_lower}
\begin{aligned}
\mathbb{P}(\boldsymbol{\xi}_{\mathrm{new}}\in A) &\ge \rho\,\mu(A),\\
\rho &\triangleq \frac{\epsilon_u\,p_l}{\mu(\mathcal{X}_{\mathrm{free}})}.
\end{aligned}
\end{equation}

Substituting this conservative bound into \eqref{eq:rn_cn} prevents the neighborhood radius $r_n$ from becoming excessively small during early iterations due to insufficient effective sampling density.
This helps maintain adequate neighbor search and rewiring opportunities throughout the tree expansion process.

\section{Proof of Theorem~\ref{thm:pc_tube_rrtstar_ao}}
\label{app:proof_thm2}
Since the neighborhood radius used in this work satisfies the asymptotic connectivity condition required by RRT$^\ast$, and the composite cost function is continuous under path perturbations and additive along trajectory segments, 
the asymptotic optimality of the proposed algorithm follows from the standard proof framework of RRT$^\ast$ \cite{2011Sampling}.

\bibliographystyle{IEEEtran}
\bibliography{refs}

\end{document}